\documentclass[11pt]{article}

\usepackage[T1]{fontenc}
\usepackage{enumitem}
\usepackage{color, colortbl}
\usepackage{amsmath, amssymb, amsthm, mathrsfs}
\usepackage{mathtools, mathabx}
\usepackage{color}
\usepackage{cancel}
\usepackage{wrapfig}
\usepackage{subcaption}
\usepackage{caption}
\usepackage{xfrac}
\usepackage{algorithm}
\usepackage{algorithmic}
\usepackage{mdframed}
\usepackage{hyperref}
\usepackage{dsfont}
\usepackage{xcolor}
\usepackage{inputenc}
\usepackage{times}
\hypersetup{
	colorlinks,
	linkcolor={red!40!gray},
	citecolor={blue!40!gray},
	urlcolor={blue!70!gray}
}
\usepackage{cleveref}

\usepackage[numbers,sort&compress]{natbib}
\usepackage{fullpage}
\newcommand{\appref}[2]{#2}

\definecolor{Gray}{gray}{0.9}
\definecolor{GrayH}{gray}{0.7}

\newcommand{\ind}{\mathds{1}}

\newcommand{\dd}{\mathrm{d}}

\newcommand{\m}[1]{#1}
\newcommand{\cd}{\overset{\mathcal{D}}{\to}}
\newtheorem{theorem}{Theorem}
\newtheorem{lemma}{Lemma}

\newtheorem*{theorem*}{Theorem}

\newtheorem*{claim*}{Claim}
\newtheorem*{fact*}{Fact}
\newtheorem*{observation*}{Observation}
\theoremstyle{definition}

\newtheorem*{remark*}{Remark}
\newtheorem{example}{Example}
\DeclareMathOperator*{\argmin}{arg\,min}

\newcommand{\iidsim}{\overset{\text{iid}}{\sim}}
\newcommand{\eqd}{\overset{\mathcal{D}}{=}}

\title{Multiplicative noise and heavy tails in stochastic optimization}
  \author{          Liam Hodgkinson \\
  ICSI and Department of Statistics\\
  University of California, Berkeley\\
  \texttt{liam.hodgkinson@berkeley.edu}
  \and
   Michael W. Mahoney\\
  ICSI and Department of Statistics\\
  University of California, Berkeley\\
  \texttt{mmahoney@stat.berkeley.edu}
 }

\begin{document}
\maketitle

\begin{abstract}

Although stochastic optimization is central to modern machine learning, the precise mechanisms underlying its success, and in particular, the precise role of the stochasticity, still remain unclear. Modelling stochastic optimization algorithms as discrete random recurrence relations, we show that multiplicative noise, as it commonly arises due to variance in local rates of convergence, results in heavy-tailed stationary behaviour in the parameters. A detailed analysis is conducted for SGD applied to a simple linear regression problem, followed by theoretical results for a much larger class of models (including non-linear and non-convex) and optimizers (including momentum, Adam, and stochastic Newton), demonstrating that our qualitative results hold much more generally. In each case, we describe dependence on key factors, including step size, batch size, and data variability, all of which exhibit similar qualitative behavior to recent empirical results on state-of-the-art neural network models from computer vision and natural language processing. Furthermore, we empirically demonstrate how multiplicative noise and heavy-tailed structure improve capacity for basin hopping and exploration of non-convex loss surfaces, over commonly-considered stochastic dynamics with only additive noise and light-tailed structure.

\end{abstract}

\section{Introduction}

Stochastic optimization is the process of minimizing a deterministic objective function via the simulation of random elements, and it is one of the most successful methods for optimizing complex or unknown objectives. 
Relatively simple stochastic optimization procedures---in particular, stochastic gradient descent (SGD)---have become the backbone of modern machine learning (ML) \cite{ma2017power}. To improve understanding of stochastic optimization in ML, and particularly why SGD and its extensions work so well, recent theoretical work has sought to study its properties and dynamics~\cite{liu2019deep}. 

Such analyses typically approach the problem through one of two perspectives. 
The first perspective, an \emph{optimization (or quenching) perspective}, examines convergence either in expectation \cite{chen2018convergence,zhou2018convergence,gower2019sgd,jain2019sgd,fontaine2020continuous} or with some positive (high) probability \cite{roosta2016sub,du2017gradient,kleinberg2018alternative,ward2018adagrad} through the lens of a deterministic counterpart. 
This perspective inherits some limitations of deterministic optimizers, including assumptions (e.g., convexity, Polyak-\L{}ojasiewicz criterion, etc.) that are either not satisfied by state-of-the-art problems, or not strong enough to imply convergence to a quality (e.g., global) optimum. 
More concerning, however, is the inability to explain what has come to be known as the ``generalization gap'' phenomenon:
increasing stochasticity by reducing batch size appears to improve generalization performance \cite{keskar2016large,martin2018implicit}. 
Empirically, existing strategies do tend to break down for inference tasks when using large batch sizes \cite{golmant2018computational}. 
The second perspective, a \emph{probabilistic (annealing) perspective}, examines algorithms through the lens of Markov process theory \cite{freidlin1998random,henderson2003theory,mandt2016variational,dieuleveut2017bridging}. 
Here, stochastic optimizers are interpreted as samplers from probability distributions concentrated around optima \cite{mandt2017stochastic}, and annealing the optimizer (by reducing step size) increasingly concentrates probability mass around global optima \cite{henderson2003theory}. 
Traditional analyses trade restrictions on the objective for precise annealing schedules that guarantee adequate mixing and ensure convergence. 
However, it is uncommon in practice to consider step size schedules that decrease sufficiently slowly as to guarantee convergence to global optima with probability one \cite{li2019budgeted}. 
In fact, SGD based methods with poor initialization can easily get stuck near poor local minima using a typical step decay schedule \cite{liu2019bad}. 

More recent efforts conduct a \emph{distributional analysis}, directly examining the probability distribution that a stochastic optimizer targets for each fixed set of hyperparameters \cite{mandt2016variational,babichev2018constant,dieuleveut2017bridging}. 
Here, one can assess a stochastic optimizer according to its capacity to reach and then occupy neighbourhoods of high-quality optima in the initial stages, where the step size is large and constant. 
As the step size is then rapidly reduced, tighter neighbourhoods with higher probability mass surrounding nearby minima are achievable. This is most easily accomplished using a variational approach by appealing to continuous-time Langevin approximations \cite{mandt2016variational, chaudhari2018stochastic}, whose stationary distributions are known explicitly \cite{ma2015complete}; but these approaches also require restrictive assumptions, such as constant or bounded volatility \cite{mandt2017stochastic,orvieto2019continuous}. 
Interestingly, these assumptions parallel the common belief that the predominant part of the stochastic component of an optimizer is an additive perturbation \cite{kleinberg2018alternative,zhang2019algorithmic}. 
Such analyses have been questioned with recent discoveries of non-Gaussian noise \cite{simsekli2019tail} that leads to \emph{heavy-tailed} stationary behaviour (whose distributions have infinite Laplace transform).
This behavior implies stronger exploratory properties and an increased tendency to rapidly reach faraway basins than earlier Langevin-centric analyses suggest; but even these recent non-Gaussian analyses have relied on strong assumptions, additive noise, and continuous-time approximations of their own \cite{simsekli2020fractional}.

\paragraph{Main Contributions.}  We model stochastic optimizers as Markov random recurrence relations, thereby avoiding continuous-time approximations, and we examine their stationary distributions.
The formulation of this model is described in \S\ref{sec:Model}.
We show that \emph{multiplicative noise}, frequently assumed away in favour of more convenient \emph{additive noise} in continuous analyses, can often lead to heavy-tailed stationary behaviour, playing a critical role in the dynamics of a stochastic optimizer, and influencing its capacity to hop between basins in the loss landscape. 
In this paper, we consider heavy-tailed behavior that assumes a power law functional form.  
We say that the stationary distributions of the parameters/weights $W$ exhibit \emph{power laws}, with tail probabilities $\mathbb{P}(\|W\| > t) = \Omega(t^{-\alpha})$ as $t \to \infty$, for some $\alpha > 0$ called the \emph{tail exponent} (where smaller tail exponents correspond to heavier tails) --- further details, including precise definitions are in Appendix \appref{A}{\ref{sec:LinearRecurrence}}.
To inform our analysis, in \S\ref{sec:LinearOptim}, we examine the special case of constant step-size SGD applied to linear least squares, and we find it obeys a \emph{random linear recurrence relation} displaying both multiplicative and additive noise. 
Using well-known results \cite{buraczewski2016stochastic}, we isolate three regimes determining the tail behaviour of SGD (shown in Table~\ref{tab:Regimes}, discussed in detail in \S\ref{sec:LinearOptim}), finding stationary behaviour \emph{always} exhibits a precise power law in an infinite data regime.
\begin{table}
\centering
\begin{tabular}{|l|c|c|c|}
\hline
\rowcolor{Gray}\textbf{Regime} & \textbf{Condition on $A$} & \textbf{Tails for $B$} & \textbf{Tails for $W_{\infty}$} \\
\hline
Light-tailed (\S\appref{A.1}{\ref{sec:GoldieGrubel}}) & $\mathbb{P}(\|A\| \leq 1) = 1$ & $\mathcal{O}(e^{-\lambda t})$ & $\mathcal{O}(e^{-\rho t})$ \\
Heavy-tailed (multiplicative) (\S\appref{A.2}{\ref{sec:KestenGoldie}}) & $\|\sigma_{\min}(A)\|_{\alpha} = 1$ & $o(t^{-\alpha})$ & $\Theta(t^{-\alpha})$ \\
Heavy-tailed (additive) (\S\appref{A.3}{\ref{sec:GrinceviciusGrey}}) & --- & $\Omega(t^{-\beta})$ & $\Omega(t^{-\beta})$ \\
\hline
\end{tabular}
\caption{\label{tab:Regimes} Summary of the three primary tail behaviour regimes for the stationary distribution of (\ref{eq:LinearRecurrence}).}
\end{table}
\begin{table}
    \centering
    \begin{tabular}{|l|c|c|c|}
    \hline
    \rowcolor{Gray}\textbf{Regime} & \textbf{Noise} & \textbf{Condition on }$r$ & \textbf{Quadratic Example} \\
\hline
        Light-tailed \cite{ma2015complete} & White noise & $\liminf_{\|w\|\to\infty} r(w) > 0$ & $g$ bounded \\
        Heavy-tailed (multiplicative) \cite{ma2015complete} & White noise & $\limsup_{\|w\|\to\infty} r(w) = 0$ & $g(w) = c \nabla f(w)$\\
        Heavy-tailed (additive) \cite{simsekli2019tail} & L\`{e}vy noise & --- & arbitrary $g$ \\
\hline
    \end{tabular}
    \caption{Summary of three tail behaviour regimes for continuous-time models (\ref{eq:LangevinModel}).}
    \label{tab:Continuous}
\end{table}
In \S\ref{sec:General}, these observations are extended by providing sufficient conditions for the existence of power laws arising in arbitrary iterative stochastic optimization algorithms on both convex and non-convex problems, with more precise results when updates are Lipschitz. Using our results, factors influencing tail behaviour are examined, with existing empirical findings supporting the hypothesis that heavier tails coincide with improved generalization performance \cite{martin2020predicting}. 
Numerical experiments are conducted in \S\ref{sec:Numerics}, illustrating how multiplicative noise and heavy-tailed stationary behaviour improve capacity for basin hopping (relative to light-tailed stationary behaviour) in the exploratory phase of learning. 
We finish by discussing impact on related work in \S\ref{sec:Discussion}, including a continuous-time analogue of Table~\ref{tab:Regimes} (shown in Table~\ref{tab:Continuous}).

\paragraph{Related Work.}  
There is a large body of related work, and we review only the most directly related here.
Analysis of stochastic optimizers via stationary distributions of Markov processes was recently considered in \cite{mandt2016variational, babichev2018constant, dieuleveut2017bridging}. 
The latter, in particular, examined first and second moments of the stationary distribution, although these can be ineffective measures of concentration in heavy-tailed settings. Heavy tails in machine learning have been observed and empirically examined in spectral distributions of weights \cite{martin2017rethinking,martin2018implicit,martin2019traditional,martin2020heavy,martin2020predicting} and in the weights themselves \cite{simsekli2019tail,panigrahi2019non}, but (ML style) theoretical analyses on the subject seem limited to continuous-time examinations \cite{simsekli2019tail,simsekli2020fractional}\footnote{Immediately prior to uploading this article to the arXiv, we became aware of \cite{gurbuzbalaban2020heavy} (submitted to the arXiv a few days ago), which conducts a detailed analysis of heavy tails in the stationary distribution of SGD applied to least-squares linear regression. Our discussion in \S\ref{sec:LinearOptim} and Appendix \appref{A}{\ref{sec:LinearRecurrence}}, which for us serves predominantly as motivation in this special case, recovers a number of their results.  Our primary contributions in \S\ref{sec:Model} and \S\ref{sec:General} provide extensions of these results to more general stochastic optimization algorithms and more general stochastic optimization problems. We encourage readers to see \cite{gurbuzbalaban2020heavy} for details in the linear least-squares~setting. }. 
Connections between multiplicative noise and heavy-tailed fluctuations can be seen throughout the wider literature \cite{deutsch1993generic, frisch1997extreme, sornette1997convergent, buraczewski2016stochastic}. From a physical point of view, multiplicative noise acts as an external environmental field, capable of exhibiting drift towards low energy states \cite{volpe2016effective}. Hysteretic optimization \cite{pal2006hysteretic} is one example of a stochastic optimization algorithm taking advantage of this property. To our knowledge, no theoretical analysis of this phenomenon has been conducted in a general optimization setting.
Indeed, while multiplicative noise in stochastic optimization has been explored in some recent empirical analyses \cite{wu2020noisy,zhang2018removing,holland2018robust}, its impact appears underappreciated, relative to the well-studied and exploited effects of additive noise  \cite{ge2015escaping,jin2017escape,du2017gradient,kleinberg2018alternative}. 
Section~\ref{sec:Discussion} contains a discussion of additional related work in light of our results.

\paragraph{Notation.}  We let $\m I_d$ denote the $d \times d$ identity matrix. Unless otherwise stated, we default to the following norms: for vector $x$, we let $\|x\| = (\sum_{i=1}^n x_i^2)^{1/2}$ denote the Euclidean norm of $x$, and for matrix $\m A$, the norm $\|\m A\| = \sup_{\|x\| = 1}\|\m A x\|$ denotes the $\ell^2$-induced (spectral) norm. Also, for matrix $\m A$, $\|\m A\|_F = \|\mathrm{vec}(\m A)\|$ is the Frobenius norm, and $\sigma_{\min}(A)$, $\sigma_{\max}(A)$ are its smallest and largest singular values, respectively. For two vectors/matrices $A,B$, we let $A\otimes B$ denote their tensor (or Kronecker) product. Knuth asymptotic notation \cite{knuth1976big} is adopted. The notation $(\Omega, \mathcal{E},\mathbb{P})$ is reserved to denote an appropriate underlying probability space. For two random elements $X,Y$, $X \eqd Y$ if $X,Y$ have the same distribution. Finally, for random vector/matrix $X$, for $\alpha > 0$, $\|X\|_\alpha = (\mathbb{E}\|X\|^{\alpha})^{1/\alpha}$. All proofs are relegated to Appendix \appref{C}{\ref{app:Proofs}}.

\section{Stochastic optimization as a Markov chain}
\label{sec:Model}
In this section, we describe how to model a stochastic optimization algorithm as a Markov chain --- in particular, a \emph{random recurrence relation}. 
This formulation is uncommon in the ML literature, but will be important for our analysis. 
Consider a general single-objective stochastic optimization problem, where the goal is to solve problems of the form $w^{\ast} = \argmin_w \mathbb{E}_\mathcal{D} \ell(w, X)$ 
for some scalar loss function $\ell$, and random element $X \sim \mathcal{D}$ (the data) \cite[\S12]{kroese2013handbook}.
In the sequel, we shall assume the weights $w$ occupy a vector space $S$ with norm $\|\cdot\|$.
To minimize $\ell$ with respect to $w$, some form of fixed point iteration is typically adopted. Supposing that there exists some continuous map $\Psi$ such that any fixed point of $\mathbb{E}\Psi(\cdot,X)$ is a minimizer of $\ell$, the sequence of iterations
\begin{equation}
\label{eq:StochOptExp}
W_{k+1} = \mathbb{E}_\mathcal{D} \Psi(W_k, X)
\end{equation}
either diverges, or converges to a minimizer of $\ell$ \cite{granas2013fixed}. In practice, this expectation may not be easily computed, so one could instead consider the sequence of iterated random functions
\begin{equation}
\label{eq:StochAlgoMarkov}
W_{k+1} = \Psi(W_k, X_k),\qquad X_k \iidsim \mathcal{D},\qquad \mbox{for }k=0,1,\dots.
\end{equation}
For example, one could consider the Monte Carlo approximation of (\ref{eq:StochOptExp}):
\begin{equation}
\label{eq:StochBatchMarkov}
W_{k+1} = \frac1n \sum_{i=1}^n \Psi(W_k, X_i^{(k)}),\quad \text{with }X_i^{(k)} \iidsim \mathcal{D},\mbox{ for } k=0,1,\dots,
\end{equation}
which can be interpreted in the context of \emph{randomized subsampling with replacement}, where each $(X_i^{(k)})_{i=1}^n$ is a batch of $n$ subsamples drawn from a large dataset $\mathcal{D}$. Subsampling without replacement can also be treated in the form (\ref{eq:StochBatchMarkov}) via the Markov chain $\{W_{sk}\}_{k=0}^{\infty}$, where $s$ is the number of minibatches in each~epoch. 

Since (\ref{eq:StochAlgoMarkov}) will not, in general, converge to a single point, the \emph{stochastic approximation (SA)} approach of Robbins--Munro \cite{robbins1951stochastic} considers a corresponding sequence of maps $\Psi_k$ given by $\Psi_k(w, x) \coloneqq (1 - \gamma_k)w + \gamma_k \Psi(w, x)$
where $\{\gamma_k\}_{k=1}^{\infty}$ is a decreasing sequence of \emph{step sizes}. Provided $\Psi$ is uniformly bounded, $\sum_k \gamma_k = \infty$ and $\sum_k \gamma_k^2 < \infty$, the sequence of iterates $W_{k+1} = n^{-1}\sum_{i=1}^n \Psi_k(W_k, X_i^{(k)})$ converges in $L^2$ and almost surely to a minimizer of $\ell$ \cite{blum1954approximation}. 
Note that this differs from the \emph{sample average approximation (SAA)} approach \cite{kleywegt2002sample}, where a deterministic optimization algorithm is applied to a random objective (e.g., gradient descent on subsampled empirical risk). 

Here are examples of how popular ML stochastic optimization algorithms fit into this framework.
\begin{example}[SGD \& SGD with momentum]
Minibatch SGD with step size $\gamma$ coincides with (\ref{eq:StochBatchMarkov}) under
$\Psi(w, X) = w - \gamma \nabla \ell(w, X).$
Incorporating momentum is possible by considering the augmented space of weights and velocities together. 
In the standard setup, letting $v$ denote velocity, and $w$ the weights, $\Psi((v,w),X) = (\eta v + \nabla \ell(w,X), w - \gamma (\eta v + \nabla \ell(w,X)))$.
\end{example}
\begin{example}[Adam]
Using state-space augmentation, the popular first-order Adam optimizer \cite{kingma2015adam} can also be cast into the form of (\ref{eq:StochAlgoMarkov}). 
Indeed, for $\beta_1,\beta_2,\epsilon > 0$, letting $g(w,X) = \sum_{i=1}^n \nabla \ell(w,X_i)$, for $M(m,w,X) = \beta_1 m + n^{-1}(1-\beta_1) g(w,X)$, $V(v,w,X) = \beta_2 v + (1-\beta_2)n^{-2} g(w,X)^2$, $B_1(b) = 1 - \beta_1(1-b)$, $B_2(b) = 1 - \beta_2(1-b)$, and
\[
W(b_1,b_2,m,v,w,X) = w - \eta \frac{M(m, w, X) / B_1(b_1)}{\sqrt{V(v, w, X) / B_2(b_2)} + \epsilon},
\]
iterations of the Adam optimizer satisfy $\mathcal{W}_{k+1} = \Psi(\mathcal{W}_k, X_k)$, where
$
\Psi((b_1,b_2,m,v,w),X) = (B_1(b_1), \allowbreak B_2(b_2), \allowbreak M(m,w,X), \allowbreak V(v,w,X), \allowbreak W(b_1,b_2,m,v,w,X)).
$
\end{example}
\begin{example}[Stochastic Newton]
The formulation (\ref{eq:StochAlgoMarkov}) is not limited to first-order stochastic optimization. Indeed, for $g(w,X) = \sum_{i=1}^n \nabla \ell(w,X_i)$ and $H(w,X) = \sum_{i=1}^n \nabla^2 \ell(w, X_i)$, the choice $\Psi(w, X) = w - \gamma H(w,X)^{-1}g(w,X)$ coincides with the stochastic Newton method \cite{roosta2016sub}.
\end{example}

The iterations (\ref{eq:StochAlgoMarkov}) are also sufficiently general to incorporate other random effects, e.g., dropout \cite{srivastava2014dropout}. 
Instead of taking the SA approach, in our analysis we will examine the original Markov chain (\ref{eq:StochAlgoMarkov}), where hyperparameters (including step size) are fixed. 
This will provide a clearer picture of ongoing dynamics at each stage within an arbitrary learning rate schedule \cite{babichev2018constant}. 
Also, assuming reasonably rapid mixing, global tendencies of a stochastic optimization algorithm, such as the degree of exploration in space, can be examined through the \emph{stationary distribution} of (\ref{eq:StochAlgoMarkov}).

\section{The linear case with SGD}
\label{sec:LinearOptim}
In this section, we consider the special case of ridge regression, i.e., least squares linear regression with $L^2$ regularization, using vanilla SGD. 
Let $\mathcal{D}$ denote a dataset comprised of inputs $x_i \in \mathbb{R}^d$ and corresponding labels $y_i \in \mathbb{R}^m$. 
Supposing that $(X,Y)$ is a pair of random vectors uniformly drawn from $\mathcal{D}$, for $\lambda \geq 0$, we seek a solution to
\begin{equation}
\label{eq:RidgeRegression}
M^{\ast} = \argmin_{M\in\mathbb{R}^{d\times m}} \tfrac12 \mathbb{E}_\mathcal{D}\|Y - M X\|^2 + \tfrac12 \lambda \|M\|_F^2.
\end{equation}
Observing that
\[
\frac{\partial}{\partial M} \left(\tfrac12 \|Y - M X\|^2 + \tfrac12 \lambda \|M\|_F^2\right) = M (X X^\top + \lambda I_d) - Y X^\top, 
\]
one can show that the solution to (\ref{eq:RidgeRegression}) is 
$$
M^{\ast} = \mathbb{E}_\mathcal{D}[YX^\top]\mathbb{E}_\mathcal{D}[XX^\top + \lambda I_d]^{-1}  . 
$$
(Note that $\lambda$ need not be strictly positive for the inverse to exist.) 
On the other hand, applying minibatch SGD to solving (\ref{eq:RidgeRegression}) with constant (sufficiently small) step size results in a Markov chain $\{M_k\}_{k=0}^{\infty}$ in the estimated parameter matrix. 
We start with the following observation, which is immediate but important.
\begin{lemma}
\label{lem:LinearRecur}
Let $n$ denote the size of each minibatch $(X_{ik},Y_{ik})_{i=1}^n$ comprised of independent and identically distributed copies of $(X,Y)$ for $k=1,2,\dots$. For $W_k$ the vectorization of $M_k$, iterations of minibatch SGD undergo the following random linear recurrence relation
\begin{equation}
\label{eq:LinearRecurrence}
W_{k+1} = \m A_k W_k + \m B_k,
\end{equation}
where $
\m A_k = I_m \otimes ((1-\lambda\gamma)\m I_d - \gamma n^{-1} \sum_{i=1}^n X_{ik}X_{ik}^\top)$, and $\m B_k = \gamma n^{-1}\sum_{i=1}^n Y_{ik} \otimes X_{ik}$. If $A_k,B_k$ are non-atomic, $\log^+\|A_k\|$, $\log^+\|B_k\|$ are integrable, and $\mathbb{E}_\mathcal{D}\log\|A_k\| < 0$, then (\ref{eq:LinearRecurrence}) is ergodic.
\end{lemma}
From here, it immediately follows that, even in this simple setting, SGD possesses \emph{multiplicative noise} in the form of the factor $A_k$, as well as \emph{additive noise} in the form of the factor $B_k$. 
Under the conditions of Lemma~\ref{lem:LinearRecur}, the expectations of (\ref{eq:LinearRecurrence}) converge to $M^{\ast}$.
Although the dynamics of this process, as well as the shape of its stationary distribution, are not as straightforward, random linear recurrence relations are among the most well-studied discrete-time processes, and as multiplicative processes, are well-known to exhibit heavy-tailed behaviour. 
In Appendix~\appref{A}{\ref{sec:LinearRecurrence}}, we discuss classical results on the topic.
The three primary tail regimes are summarized in Table~\ref{tab:Regimes}, where each $\alpha,\beta,\lambda,\rho$ denotes a strictly positive value. In particular, we have the following result.
\begin{lemma}
\label{lem:LinearHeavy}
Assume that the distribution of $X$ has full support on $\mathbb{R}^d$. Then any stationary distribution of (\ref{eq:LinearRecurrence}) is heavy-tailed.
\end{lemma}
Lemma~\ref{lem:LinearHeavy} suggests that heavy-tailed fluctuations could be more common than previously considered.
This is perhaps surprising, given that common Langevin approximations of SGD applied to this problem exhibit light-tailed stationary distributions \cite{mandt2016variational,orvieto2019continuous}. 
In (\ref{eq:LinearRecurrence}), light-tailed stationary distributions can only occur when $X$ does not have full support on $\mathbb{R}^d$ and the step size is taken sufficiently small. 
Referring to Table~\ref{tab:Regimes}, there are two possible mechanisms by which the stationary distribution of (\ref{eq:LinearRecurrence}) can be heavy-tailed. 
Most discussions about SGD focus on the additive noise component $B_k$ \cite{kleinberg2018alternative} (presumably since the analysis is simpler).
In this case, if $B_k$ is heavy-tailed, then the stationary distribution of (\ref{eq:LinearRecurrence}) is also heavy-tailed. 
This is the assumption considered in \cite{simsekli2019tail,simsekli2020fractional}. 

However, this is not the only way heavy-tailed noise can arise. 
Indeed, the multiplicative heavy-tailed regime illustrates how a power law can arise from light-tailed data (in fact, even from data with finite support). 
One of the most important attributes here is that the tail exponent of the stationary distribution is entirely due to the recursion and properties of the multiplicative noise $A_k$. 
In fact, this is true even if $B_k$ displays (not too significant) heavy-tailed behaviour. 
Here, the heavy-tailed behaviour arises due to \emph{intrinsic factors} to the stochastic optimization, and they tend to dominate over time. 
Later, we shall go further and analyze factors influencing the heaviness of these tails, or more precisely, its \emph{tail exponent}, denoted by $\alpha$ in Table~\ref{tab:Regimes}. 
Evidently, $\alpha$ is dependent on the dispersal of the distribution of $A_k$, which is itself dependent (in differing ways) on the step size $\gamma$, batch size $n$, and the dispersal and condition number of the covariance matrices $X X^\top$ of the input data.
These three factors have been noted to be integral to the practical performance of SGD \cite{smith2018disciplined,xing2018walk}. 

\section{Power laws for general objectives and optimizers}\label{sec:General}

In this section, we consider the general case (\ref{eq:StochAlgoMarkov}), and we examine how heavy-tailed stationary behaviour can arise for any stochastic optimizer, applied to both convex and non-convex problems.
As in Section~\ref{sec:LinearOptim}, here we are most interested in the presence of heavy-tailed fluctuations due to multiplicative factors. 
(The case when heavy-tailed fluctuations arise from the additive noise case is clear: if, for every $w \in \mathbb{R}^d$, $\Psi(w,X)$ is heavy-tailed/has infinite $\alpha$th-moment, then for each $k=1,2,\dots$, $W_k$ is heavy-tailed/has infinite $\alpha$th moment also, irrespective of the dynamics of $W$.)

In our main result (Theorem~\ref{thm:Lipschitz}), we illustrate how power laws arise in general Lipschitz stochastic optimization algorithms that are contracting on average and also strongly convex near infinity with positive probability. \begin{theorem}
\label{thm:Lipschitz}
Let $(S,\|\cdot\|)$ be a separable Banach space. Assume that $\Psi: S \times \Omega \to S$ is a random function on $S$ such that $\Psi$ is almost surely Lipschitz continuous and the probability measure of $\Psi$ is non-atomic. We let $K_\Psi$ denote a random variable such that $K_\Psi(\omega)$ is the Lipschitz constant of $\Psi(\cdot,\omega)$ for each $\omega \in \Omega$. Assume that $K_\Psi$ is integrable, and $\|\Psi(w)- w\|$ is integrable for some $w^{\ast} \in S$. 
Suppose there exist non-negative random variables $k_\Psi, M_\Psi$ such that, with probability one,
\begin{equation}
\label{eq:LipschitzOptim}
k_\Psi \|w - w^{\ast}\| - M_\Psi \leq \|\Psi(w) - \Psi(w^{\ast})\| \leq K_\Psi \|w - w^{\ast}\|,\qquad \mbox{for all }w \in S.
\end{equation}
If $\mathbb{E}\log K_{\Psi} < 0$, then the Markov chain given by $W_{k+1} = \Psi_k(W_k), k=0,1,\dots,$
where each $\Psi_k$ is an independent and identically distributed copy of $\Psi$, is geometrically ergodic with $W_{\infty} = \lim_{k\to\infty} W_k$ satisfying the distributional fixed point equation $W_{\infty} \eqd \Psi(W_{\infty})$. Furthermore, if $k_\Psi > 1$ with positive probability and $\|\Psi(w^\ast)\|_\alpha<\infty$ for any $\alpha > 0$, then:
\begin{enumerate}[leftmargin=*]
\item There exist $\mu,\nu,C_\mu,C_\nu > 0$ such that
$
C_\mu (1+t)^{-\mu} \leq \mathbb{P}(\|W_{\infty}\| > t) \leq C_\nu t^{-\nu},\mbox{ for all }t > 0.
$
\item There exist $\alpha,\beta > 0$ such that $\|K_{\Psi}\|_{\beta} = 1$ and $\|k_{\Psi}\|_{\alpha} = 1$ and for $\delta > 0$,
\[
0 < \limsup_{t\to\infty} t^{\alpha+\epsilon}\mathbb{P}(\|W_{\infty}\| > t),\qquad \mbox{and}\qquad \limsup_{t\to\infty} t^{\beta-\epsilon}\mathbb{P}(\|W_{\infty}\| > t) < +\infty.
\]
\item For any $p \in (0,\beta)$ and $k=1,2,\dots$, the moments of $\|W_k\|_p$ are bounded above and below by 
\[
\|k_\Psi\|_p^k (\|W_0\|_p + \kappa_p^-) - \kappa_p^- \leq \|W_k\|_p \leq \|K_\Psi\|_p^k (\|W_0\|_p - \kappa_p^+)+\kappa_p^+,\]
implying $\|W_\infty\|_p \leq \kappa_p^+$, where
\[
\kappa_{p}^{-}=\frac{\| k_{\Psi}\| _{p}\| w^{\ast}\| +\| \Psi(w^{\ast})\| _{p}}{1-\| k_{\Psi}\| _{p}},\qquad\mbox{and}\qquad\kappa_{p}^{+}=\frac{\| K_{\Psi}\| _{p}\| w^{\ast}\| +\| \Psi(w^{\ast})\| _{p}}{1-\| K_{\Psi}\| _{p}}.
\]
\end{enumerate}
\end{theorem}
Geometric rates of convergence in the Prohorov and total variation metrics to stationarity are discussed in \cite{diaconis1999iterated,alsmeyer2003harris}.
From Theorem~\ref{thm:Lipschitz}, we find that the presence of expanding multiplicative noise implies the stationary distribution of (\ref{eq:StochAlgoMarkov}) is stochastically bounded between two power laws. In particular, this suggests that Lipschitz stochastic optimizers satisfying (\ref{eq:LipschitzOptim}) and $\mathbb{P}(k_\Psi > 1) > 0$ can conduct wide exploration of the loss landscape. As an example, for SGD satisfying (\ref{eq:LipschitzOptim}), we have that 
\begin{equation}
\label{eq:kPsiSGD}
k_\Psi = \liminf_{\|w\|\to\infty}\sigma_{\min}(I-\gamma\nabla^2\ell(w,X)),\qquad \mbox{and} \qquad K_\Psi = \sup_w \|I - \gamma\nabla^2\ell(w,X)\|.
\end{equation}
Therefore, for the linear case (\ref{eq:LinearRecurrence}), using Theorem \ref{thm:Lipschitz}, we can recover the same tail exponent seen in Theorem \ref{thm:Kesten}. To treat other stochastic optimizers that are not Lipschitz, or do not satisfy the conditions of Theorem~\ref{thm:Lipschitz}, we present in Lemma~\ref{lem:Abstract} an abstract sufficient condition for heavy-tailed stationary distributions of ergodic Markov chains. 
\begin{lemma}
\label{lem:Abstract}
For a general metric space $S$, suppose that $W$ is an ergodic Markov chain on $S$ with $W_k \cd W_{\infty}$ as $k \to \infty$, and let $f: S \to \mathbb{R}$ be some scalar-valued function. If there exists some $\epsilon > 0$ such that $\inf_{w \in S}\mathbb{P}(|f(\Psi(w))| > (1 + \epsilon)|f(w)|) > 0$, then $f(W_\infty)$ is heavy-tailed. 
\end{lemma}
Under our model, Lemma~\ref{lem:Abstract} says that convergent constant step-size stochastic optimizers will exhibit heavy-tailed stationary behaviour if there is some positive probability of the optimizer moving further away from an optimum, irrespective of how near or far you are from it. Analyses concerning stochastic optimization algorithms almost always consider rates of contraction towards a nearby optimum, quantifying the \emph{exploitative} behaviour of a stochastic optimizer. To our knowledge, this is the first time that consequences of the fact that stochastic optimization algorithms could, at any stage, move away from any optimum, have been examined. In the convex setting, this appears detrimental, but this is critical in non-convex settings, where \emph{exploration} (rather than exploitation to a local optimum) is important. Indeed, we find that it is this behaviour that directly determines the tails of the stochastic optimizer's stationary distribution, and therefore, its exploratory behaviour. Altogether, our results suggest that heavy-tailed noise could be more common and more beneficial in stochastic optimization than previously acknowledged. 

\paragraph{Factors influencing the tail exponent.}  Recent analyses have associated heavier-tailed structures (i.e., larger tail exponents) with improved generalization performance \cite{martin2018implicit,martin2020predicting,simsekli2019tail}. From the bounds in Theorem~\ref{thm:Lipschitz}, and established theory for the special linear case (\ref{eq:LinearRecurrence}) (see Appendix \appref{A.2}{\ref{sec:KestenGoldie}}), the following factors play a role in \emph{decreasing the tail exponent} (generally speaking, increasing $k_\Psi,K_\Psi$ in expectation or dispersion implies decreased $\alpha$) resulting in \emph{heavier-tailed~noise}. 
\begin{itemize}
\item 
\textbf{Decreasing batch size:} 
For fixed hyperparameters, $\Psi$ becomes less deterministic as the batch size decreases, resulting in an decreased (i.e., heavier) tail exponent for ergodic stochastic optimizers. This is in line with \cite{yao2018hessian,martin2018implicit,martin2020heavy,martin2020predicting}, and it suggests a relationship to the \emph{generalization gap} phenomenon.\item
\textbf{Increasing step size:}
Also in line with \cite{martin2018implicit,martin2020heavy,martin2020predicting}, increased step sizes exacerbate fluctuations and thus result in increased tail exponents. 
However, step sizes also affect stability of the algorithm. 
The relationship between step and batch sizes has received attention \cite{balles2016coupling,smith2017don}; choosing these parameters to increase heavy-tailed fluctuations while keeping variance sensible could yield a valuable exploration~strategy. 
\item 
\textbf{More dispersed data:} 
Increasing dispersion in the data implies increased dispersion for the distribution of $k_\Psi,K_\Psi$, and hence heavier tails. For the same model trained to different datasets, a smaller tail exponent may be indicative of richer data, not necessarily of higher variance, but exhibiting a larger moment of some order. Data augmentation is a strategy to achieve this \cite{perez2017effectiveness}. 
\item 
\textbf{Increasing regularization:} 
The addition of an explicit $L^2$-regularizer to the objective function (known to help avoiding bad minima \cite{liu2019bad}) results in larger $k_\Psi$, and hence, heavier-tailed noise. 
\item 
\textbf{Increasing average condition number of Hessian: } 
The Hessian is known to play an important role in the performance of stochastic optimizers \cite{yao2018hessian,xing2018walk}. Evidently seen in (\ref{eq:kPsiSGD}) in the case of SGD, and potentially more generally, a larger condition number of the Hessian can allow for larger tail exponents while remaining ergodic. 
\end{itemize}
In each case (excluding step size, which is complicated since step size also affects stability), the literature supports the hypothesis that factors influencing heavier tails coincide with improved generalization performance; see, e.g.,~\cite{martin2018implicit} and references therein. Vanilla SGD and stochastic Newton both exhibit heavy-tailed behaviour when $\nabla^2 \ell(w,X)$ and the Jacobian of $H(w,X)^{-1}g(w,X)$ have unbounded spectral distributions, respectively (as might be the case when $X$ has full support on $\mathbb{R}^d$). However, adaptive optimizers such as momentum and Adam incorporate geometric decay that can prevent heavy-tailed fluctuations, potentially limiting exploration while excelling at exploitation to nearby optima. It has recently been suggested that these adaptive aspects should be turned off during an initial warmup phase \cite{liu2019variance}, implying that exacerbating heavy-tailed fluctuations and increased tail exponents could aid exploratory behaviour in the initial stages of training. 

\section{Numerical experiments}
\label{sec:Numerics}

To illustrate the advantages that multiplicative noise and heavy tails offer in non-convex optimization, in particular for exploration (of the entire loss surface) versus exploitation (to a local optimum), we first consider stochastic optimization algorithms in the non-stationary regime. 
To begin, in one dimension, we compare perturbed gradient descent (GD) \cite{jin2017escape} with additive light-tailed and heavy-tailed noise against a version with additional multiplicative noise.
That is,
\[
\mbox{(a,b)}\quad w_{k+1} = w_k - \gamma (f'(w_k) + (1+\sigma) Z_k),\quad \mbox{vs.}\quad \mbox{(c)}\quad w_{k+1} = w_k - \gamma ((1+\sigma Z_k^{(1)})f'(w_k) + Z_k^{(2)}),
\]
where $f:\,\mathbb{R}\to\mathbb{R}$ is the objective function, $\gamma,\sigma > 0$, $Z_k,Z_k^{(1)},Z_k^{(2)} \iidsim \mathcal{N}(0,1)$ in (a) and (c), and in (b), $Z_k$ are i.i.d. $t$-distributed (heavy-tailed) with 3 degrees of freedom, normalized to have unit variance. Algorithms (a)-(c) are applied to the objective $f(x) = \frac1{10}x^2 + 1 - \cos(x^2)$, which achieves a global (wide) minimum at zero. Iterations of (a)-(c) have expectation on par with classical GD. For fixed step size $\gamma = 10^{-2}$ and initial $w_0 = -4.75$, the distribution of $10^6$ successive iterates are presented in Figure~\ref{fig:BasinMulAdd} for small ($\sigma = 2$), moderate ($\sigma = 12$), and strong ($\sigma = 50$) noise. Both (b) and (c) readily enable jumps between basins. However, while (a), (b) smooth the effective loss landscape, making it easier to jump between basins, it also has the side effect of reducing resolution in the vicinity of minima. On the other hand, (c) maintains close proximity (peaks in the distribution) to critical points.

Figure~\ref{fig:BasinMulAdd} also illustrates exploration/exploitation benefits of multiplicative noise (and associated heavier tails) over additive noise (and associated lighter tails). While rapid exploration may be achieved using heavy-tailed additive noise \cite{simsekli2019tail}, since reducing the step size may not reduce the tail exponent, efficient exploitation of local minima can become challenging, even for small step sizes. On the other hand, multiplicative noise has the benefit of behaving similarly to Gaussian additive noise for small step sizes \cite{rubin2014mapping}.  We can see evidence of this behaviour in the leftmost column in Figure~\ref{fig:BasinMulAdd}, where the size of the multiplicative noise is small. As the step size is annealed to zero, multiplicative noise resembles the convolutional nature of additive noise~\cite{kleinberg2018alternative}.
\begin{figure}
\centering
\includegraphics[width=0.95\textwidth]{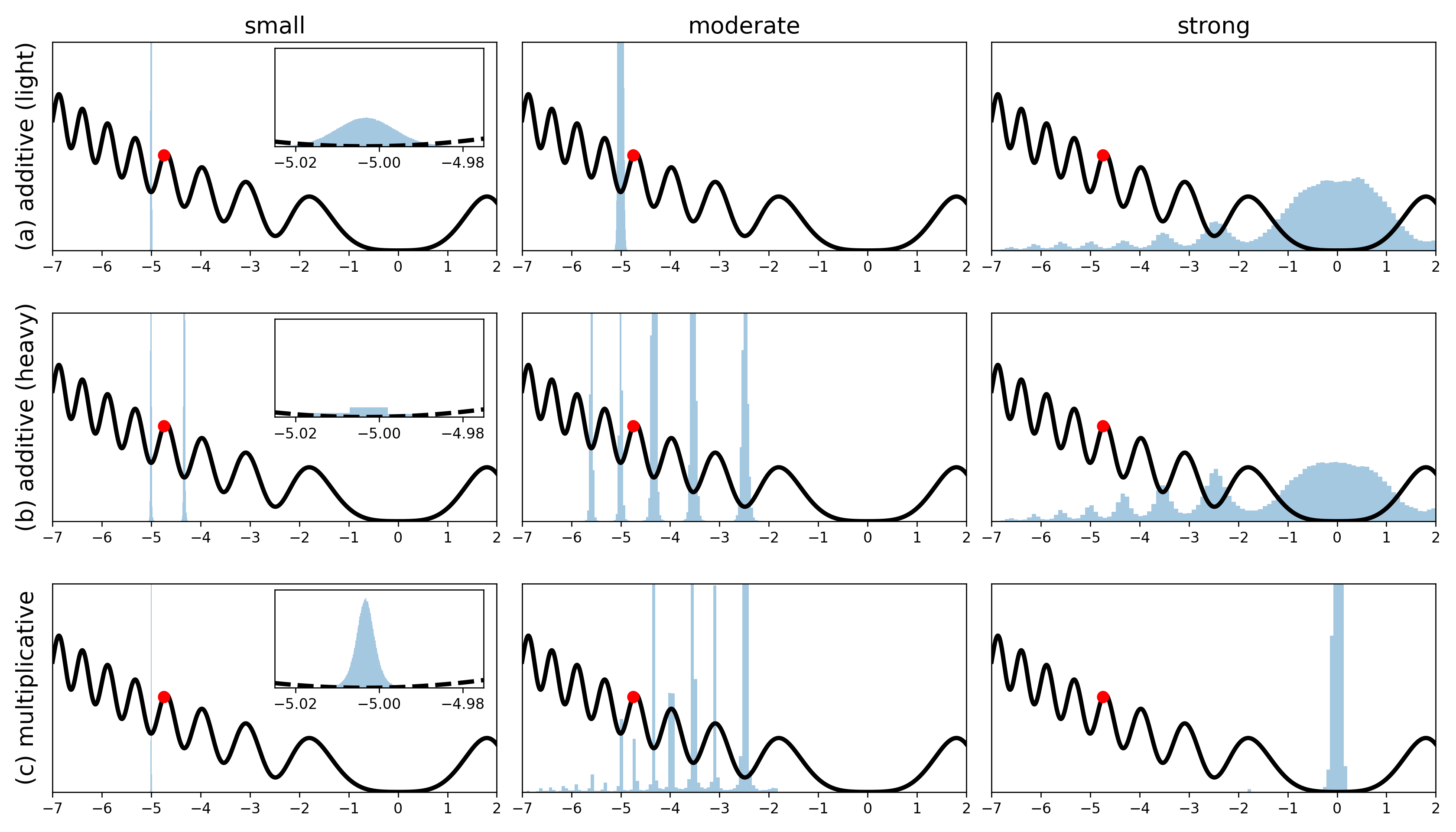}
\caption{\label{fig:BasinMulAdd}Histograms (blue) of $10^6$ iterations of GD with combinations of small, moderate, and strong versus light additive, heavy additive, and multiplicative noise, applied to a non-convex objective (black). (Red) Initial starting location for the optimization.}
\end{figure}
\begin{figure}
\centering
\includegraphics[width=\textwidth]{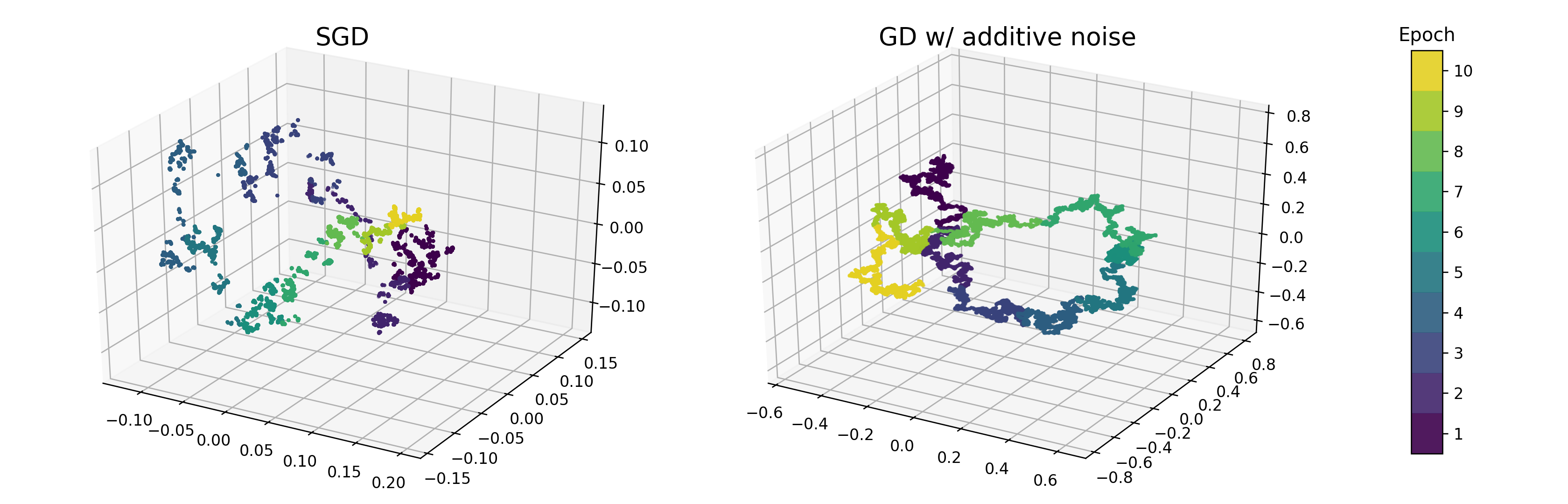}
\caption{\label{fig:BasinSGD} Heuristic visualization.  PCA scores for trajectories of SGD over 10 epochs versus perturbed GD with additive noise across the same number of iterations along principle components $2,3,4$.} \end{figure}

This same basin hopping behaviour can be observed heuristically in SGD for higher-dimensional models with the aid of principal component analysis (PCA), although jumps become much more frequent. To see this, we consider fitting a two-layer neural network with 16 hidden units for classification of the Musk dataset \cite{dietterich1997solving} (168 attributes; 6598 instances) with cross-entropy loss without regularization and step size $\gamma = 10^{-2}$. Two stochastic optimizers are compared: (a) SGD with a single sample per batch (without replacement), and (b) perturbed GD \cite{jin2017escape}, where covariance of iterations in (b) is chosen to approximate that of (a) on average.
PCA-projected trajectories are presented in Figure~\ref{fig:BasinSGD}.
SGD frequently exhibits jumps between basins, a feature not shared by perturbed GD with only additive noise. Further numerical analysis is conducted in Appendix~\appref{B}{\ref{sec:Numerics2}} to support claims in \S\ref{sec:General}.

\section{Discussion}
\label{sec:Discussion}

\paragraph{Lazy training and implicit renewal theory.}
Theory concerning random linear recurrence relations can be extended directly to SGD applied to objective functions whose gradients are ``almost'' linear, such as those seen in lazy training \cite{chizat2019lazy}, due to the implicit renewal theory of Goldie \cite{goldie1991implicit}. However, it seems unlikely that multilayer networks should exhibit the same tail exponent between each layer, where the conditions of Breiman's lemma (see Appendix~\appref{A}{\ref{sec:LinearRecurrence}}) break down. Significant discrepancies between tail exponents across layers were observed in other appearances of heavy-tailed noise \cite{martin2020predicting,simsekli2019tail}. 
From the point of view of tail exponents, it appears that most practical finite-width deep neural networks do not exhibit lazy training. This observation agrees with the poor relative generalization performance exhibited by linearized neural networks~\cite{chizat2019lazy,oymak2019generalization}.

\paragraph{Continuous-time models.}
Probabilistic analyses of stochastic optimization often consider continuous-time approximations, e.g., for constant step size, a stochastic differential equation of the form \cite{mandt2016variational, orvieto2019continuous,fontaine2020continuous}
\begin{equation}
\label{eq:LangevinModel}
\dd W_t = -\gamma \nabla f(W_t) \dd t + g(W_t) \dd X_t,
\end{equation}
where $f:\mathbb{R}^d \to \mathbb{R}$, $g:\mathbb{R}^{d} \to\mathbb{R}^{d\times d}$. The most common of these are Langevin models where $X$ is Brownian motion (white noise), although heavy-tailed noise has also been considered \cite{simsekli2019tail,simsekli2020fractional}. In the case where $g$ is diagonal, as a corollary of \cite{rubin2014mapping,ma2015complete}, we may determine conditions for heavy-tailed stationary behaviour. Letting $r(w) = \|g(w)^{-2} \nabla f(w)\|$, in Table~\ref{tab:Continuous} we present a continuous time analogue of Table~\ref{tab:Regimes} with examples of choices of $g$ that result in each regime when $f$ is quadratic.
Langevin algorithms commonly seen in the literature \cite{mandt2016variational,orvieto2019continuous} assume constant or bounded volatility ($g$), resulting in light-tailed stationary distributions for $L^2$-regularized loss functions. However, even a minute presence of multiplicative noise (unbounded $g$) can yield a heavy-tailed stationary distribution \cite{biro2005power}. Fortunately, more recent analyses are allowing for the presence of multiplicative noise in their approximations \cite{fontaine2020continuous}.
Based on our results, it seems clear that this trend is critical for adequate investigation of stochastic optimizers using the continuous-time approach. 
On the other hand, \cite{simsekli2019tail,simsekli2020fractional} invoked the generalized central limit theorem (CLT) to propose a continuous-time model with heavy-tailed (additive) L\`{e}vy noise, in response to the observation that stochastic gradient noise is frequently heavy-tailed.In their model, the heaviness of the noise does not change throughout the optimization. Furthermore, \cite{panigrahi2019non} observed that stochastic gradient noise is typically Gaussian for large batch sizes, which is incompatible with generalized CLT. We have illustrated how multiplicative noise also yields heavy-tailed fluctuations. Alternatively, heavy tails in the stochastic gradient noise for \emph{fixed} weights could be log-normal instead, which are known to arise in deep neural networks \cite{hanin2019products}. A notable consequence of the L\`{e}vy noise model is that exit times for basin hopping are essentially independent of basin height, and dependent instead on basin \emph{width}, which is commonly observed with SGD \cite{xing2018walk}. In the absence of heavy-tailed additive noise, we observed similar behaviour in our experiments for large multiplicative noise, although precise theoretical treatment remains the subject of future work.

\paragraph{Stochastic gradient MCMC.}
The same principles discussed here also apply to stochastic gradient MCMC (SG-MCMC) algorithms \cite{ma2015complete}. 
For example, the presence of multiplicative noise could suggest why SG-MCMC achieves superior mixing rates to classical Metropolis--Hastings MCMC in high dimensions, since the latter is known to struggle with multimodality, although there is some skepticism in this regard \cite{wenzel2020good}. Theorem~\ref{thm:Lipschitz} implies SG-MCMC yields excellent samples for satisfying the tail growth rate conditions required in importance sampling; see \cite{hodgkinson2020reproducing} for example.

\paragraph{Geometric properties of multiplicative noise.}
It has been suggested that increased noise in SGD acts as a form of convolution, smoothing the effective landscape \cite{kleinberg2018alternative}. This appears to be partially true. As seen in Figure~\ref{fig:BasinMulAdd}, smoothing behaviour is common for additive noise, and reduces resolution in the troughs. Multiplicative noise, which SGD also exhibits, has a different effect. As seen in \cite{rubin2014mapping}, in the continuous-time setting, multiplicative noise equates to conducting additive noise on a modified (via Lamperti transform) loss landscape. There is also a natural interpretation of multiplicative noise through choice of geometry in Riemannian Langevin diffusion \cite{girolami2011riemann,ma2015complete}. Under either interpretation, it appears multiplicative noise shrinks the width of peaks and widens troughs in the effective loss landscape, potentially negating some of the undesirable side effects of additive noise.

\paragraph{Heavy-tailed mechanistic universality.}
In a recent series of papers \cite{martin2018implicit, martin2019traditional, martin2020heavy, martin2020predicting}, an empirical (or phenomenological)
theory of \emph{heavy-tailed self-regularization} is developed, proposing that sufficiently complex and well-trained deep neural networks exhibit \emph{heavy-tailed mechanistic universality}: the spectral distribution of large weight matrices display a power law whose tail exponent is negatively correlated with generalization performance. If true, examination of this tail exponent provides an indication of model quality, and factors that are positively associated with improved test performance, such as decreased batch size. From random matrix theory, these heavy-tailed spectral distributions may arise either due to strong correlations arising in the weight matrices, or heavy-tailed distributions arising in each of the weights over time. The reality may be some combination of both; here, we have illustrated that the second avenue is at least possible, and relationships between factors of the optimization and the tail exponent agree with the present findings. Proving that heavy-tailed behaviour can arise in spectral distributions of the weights (as opposed to their probability distributions, which we have treated here) is a challenging problem, and remains the subject of future work. Casting the evolution of spectral distributions into the framework of iterated random functions could prove fruitful in this respect. 

\section{Conclusion}
A theoretical analysis on the relationship between multiplicative noise and heavy-tailed fluctuations in stochastic optimization algorithms has been conducted. We propose that viewing stochastic optimizers through the lens of Markov process theory and examining stationary behaviour is key to understanding exploratory behaviour on non-convex landscapes in the initial phase of training. Our results suggest that heavy-tailed fluctuations may be more common than previous analyses have suggested, and they further the hypothesis that such fluctuations are correlated to improved generalization performance. From this viewpoint, we maintain that multiplicative noise should not be overlooked in future analyses on the subject. 

\paragraph{Acknowledgments.} We would like to acknowledge DARPA, NSF, and ONR for providing partial support of this work.

\bibliographystyle{abbrv}

\begin{thebibliography}{10}

\bibitem{alsmeyer2003harris}
G.~Alsmeyer.
\newblock {On the Harris recurrence of iterated random Lipschitz functions and
  related convergence rate results}.
\newblock {\em Journal of Theoretical Probability}, 16(1):217--247, 2003.

\bibitem{alsmeyer2016stationary}
G.~Alsmeyer.
\newblock {On the stationary tail index of iterated random Lipschitz
  functions}.
\newblock {\em Stochastic Processes and their Applications}, 126(1):209--233,
  2016.

\bibitem{alstott2014powerlaw}
J.~Alstott and D.~P. Bullmore.
\newblock {powerlaw: a Python package for analysis of heavy-tailed
  distributions}.
\newblock {\em PLOS One}, 9(1), 2014.

\bibitem{babichev2018constant}
D.~Babichev and F.~Bach.
\newblock {Constant step size stochastic gradient descent for probabilistic
  modeling}.
\newblock {\em Proceedings of the 34th Conference on Uncertainty in Artificial
  Intelligence (UAI 2018)}, 2018.

\bibitem{balles2016coupling}
L.~Balles, J.~Romero, and P.~Hennig.
\newblock {Coupling adaptive batch sizes with learning rates}.
\newblock {\em Proceedings of the 33rd Conference on Uncertainty in Artificial
  Intelligence (UAI 2017)}, 2017.

\bibitem{bingham1989regular}
N.~H. Bingham, C.~M. Goldie, and J.~L. Teugels.
\newblock {\em {Regular variation}}, volume~27.
\newblock Cambridge University Press, 1989.

\bibitem{biro2005power}
T.~S. Bir{\'o} and A.~Jakov{\'a}c.
\newblock {Power-law tails from multiplicative noise}.
\newblock {\em Physical Review Letters}, 94(13):132302, 2005.

\bibitem{blum1954approximation}
J.~R. Blum.
\newblock {Approximation methods which converge with probability one}.
\newblock {\em The Annals of Mathematical Statistics}, 25(2):382--386, 1954.

\bibitem{buraczewski2016stochastic}
D.~Buraczewski, E.~Damek, and T.~Mikosch.
\newblock {\em {Stochastic models with power-law tails}}.
\newblock Springer, 2016.

\bibitem{chaudhari2018stochastic}
P.~Chaudhari and S.~Soatto.
\newblock {Stochastic gradient descent performs variational inference,
  converges to limit cycles for deep networks}.
\newblock In {\em 2018 Information Theory and Applications Workshop (ITA)},
  pages 1--10. IEEE, 2018.

\bibitem{chen2018convergence}
X.~Chen, S.~Liu, R.~Sun, and M.~Hong.
\newblock {On the convergence of a class of Adam-type algorithms for non-convex
  optimization}.
\newblock {\em Proceedings of the 7th International Conference on Learning
  Representations (ICLR 2019)}, 2019.

\bibitem{chizat2019lazy}
L.~Chizat, E.~Oyallon, and F.~Bach.
\newblock {On lazy training in differentiable programming}.
\newblock In {\em Advances in Neural Information Processing Systems}, pages
  2933--2943, 2019.

\bibitem{clauset2009power}
A.~Clauset, C.~R. Shalizi, and M.~E. Newman.
\newblock {Power-law distributions in empirical data}.
\newblock {\em SIAM Review}, 51(4):661--703, 2009.

\bibitem{cortez2009modeling}
P.~Cortez, A.~Cerdeira, F.~Almeida, T.~Matos, and J.~Reis.
\newblock Modeling wine preferences by data mining from physicochemical
  properties.
\newblock {\em Decision Support Systems}, 47(4):547--553, 2009.

\bibitem{deutsch1993generic}
J.~M. Deutsch.
\newblock Generic behavior in linear systems with multiplicative noise.
\newblock {\em Physical Review E}, 48(6):R4179, 1993.

\bibitem{diaconis1999iterated}
P.~Diaconis and D.~Freedman.
\newblock {Iterated random functions}.
\newblock {\em SIAM Review}, 41(1):45--76, 1999.

\bibitem{dietterich1997solving}
T.~G. Dietterich, R.~H. Lathrop, and T.~Lozano-P{\'e}rez.
\newblock {Solving the multiple instance problem with axis-parallel
  rectangles}.
\newblock {\em Artificial Intelligence}, 89(1-2):31--71, 1997.

\bibitem{dieuleveut2017bridging}
A.~Dieuleveut, A.~Durmus, and F.~Bach.
\newblock {Bridging the gap between constant step size stochastic gradient
  descent and Markov chains}.
\newblock {\em arXiv preprint arXiv:1707.06386}, 2017.

\bibitem{du2017gradient}
S.~S. Du, C.~Jin, J.~D. Lee, M.~I. Jordan, A.~Singh, and B.~Poczos.
\newblock {Gradient descent can take exponential time to escape saddle points}.
\newblock In {\em Advances in Neural Information Processing Systems}, pages
  1067--1077, 2017.

\bibitem{fontaine2020continuous}
X.~Fontaine, V.~De~Bortoli, and A.~Durmus.
\newblock {Continuous and Discrete-Time Analysis of Stochastic Gradient Descent
  for Convex and Non-Convex Functions}.
\newblock {\em arXiv preprint arXiv:2004.04193}, 2020.

\bibitem{freidlin1998random}
M.~I. Freidlin and A.~D. Wentzell.
\newblock {Random perturbations}.
\newblock In {\em Random perturbations of dynamical systems}, pages 15--43.
  Springer, 1998.

\bibitem{frisch1997extreme}
U.~Frisch and D.~Sornette.
\newblock {Extreme deviations and applications}.
\newblock {\em Journal de Physique I}, 7(9):1155--1171, 1997.

\bibitem{ge2015escaping}
R.~Ge, F.~Huang, C.~Jin, and Y.~Yuan.
\newblock {Escaping from saddle points—online stochastic gradient for tensor
  decomposition}.
\newblock In {\em JMLR: Workshop and Conference Proceedings}, volume~40, pages
  797--842, 2015.

\bibitem{girolami2011riemann}
M.~Girolami and B.~Calderhead.
\newblock {Riemann manifold Langevin and Hamiltonian Monte Carlo methods}.
\newblock {\em Journal of the Royal Statistical Society: Series B (Statistical
  Methodology)}, 73(2):123--214, 2011.

\bibitem{goldie1991implicit}
C.~M. Goldie.
\newblock {Implicit renewal theory and tails of solutions of random equations}.
\newblock {\em The Annals of Applied Probability}, 1(1):126--166, 1991.

\bibitem{goldie1996perpetuities}
C.~M. Goldie and R.~Gr{\"u}bel.
\newblock {Perpetuities with thin tails}.
\newblock {\em Advances in Applied Probability}, 28(2):463--480, 1996.

\bibitem{golmant2018computational}
N.~Golmant, N.~Vemuri, Z.~Yao, V.~Feinberg, A.~Gholami, K.~Rothauge, M.~W.
  Mahoney, and J.~Gonzalez.
\newblock {On the computational inefficiency of large batch sizes for
  stochastic gradient descent}.
\newblock {\em arXiv preprint arXiv:1811.12941}, 2018.

\bibitem{gower2019sgd}
R.~M. Gower, N.~Loizou, X.~Qian, A.~Sailanbayev, E.~Shulgin, and
  P.~Richt{\'a}rik.
\newblock {SGD: General analysis and improved rates}.
\newblock {\em Proceedings of the 36th International Conference on Machine
  Learning (ICML 2019)}, 2019.

\bibitem{granas2013fixed}
A.~Granas and J.~Dugundji.
\newblock {\em {Fixed point theory}}.
\newblock Springer Science \& Business Media, 2013.

\bibitem{grey1994regular}
D.~R. Grey.
\newblock {Regular variation in the tail behaviour of solutions of random
  difference equations}.
\newblock {\em The Annals of Applied Probability}, pages 169--183, 1994.

\bibitem{grincevicius1975one}
A.~K. Grincevi{\v{c}}ius.
\newblock {One limit distribution for a random walk on the line}.
\newblock {\em Lithuanian Mathematical Journal}, 15(4):580--589, 1975.

\bibitem{gurbuzbalaban2020heavy}
M.~G{\"u}rb{\"u}zbalaban, U.~{\c{S}}im{\c{s}}ekli, and L.~Zhu.
\newblock {The Heavy-Tail Phenomenon in SGD}.
\newblock {\em arXiv preprint arXiv:2006.04740}, 2020.

\bibitem{hanin2019products}
B.~Hanin and M.~Nica.
\newblock {Products of many large random matrices and gradients in deep neural
  networks}.
\newblock {\em Communications in Mathematical Physics}, pages 1--36, 2019.

\bibitem{henderson2003theory}
D.~Henderson, S.~H. Jacobson, and A.~W. Johnson.
\newblock {The theory and practice of simulated annealing}.
\newblock In {\em Handbook of Metaheuristics}, pages 287--319. Springer, 2003.

\bibitem{hodgkinson2020reproducing}
L.~Hodgkinson, R.~Salomone, and F.~Roosta.
\newblock {The reproducing Stein kernel approach for post-hoc corrected
  sampling}.
\newblock {\em arXiv preprint arXiv:2001.09266}, 2020.

\bibitem{holland2018robust}
M.~J. Holland.
\newblock {Robust descent using smoothed multiplicative noise}.
\newblock {\em Proceedings of the 22nd International Conference on Artificial
  Intelligence and Statistics (AISTATS 2019)}, 2019.

\bibitem{jin2017escape}
C.~Jin, R.~Ge, P.~Netrapalli, S.~M. Kakade, and M.~I. Jordan.
\newblock {How to escape saddle points efficiently}.
\newblock In {\em Proceedings of the 34th International Conference on Machine
  Learning (ICML 2017)}, pages 1724--1732. JMLR.org, 2017.

\bibitem{keskar2016large}
N.~S. Keskar, D.~Mudigere, J.~Nocedal, M.~Smelyanskiy, and P.~T.~P. Tang.
\newblock {On large-batch training for deep learning: generalization gap and
  sharp minima}.
\newblock {\em Proceedings of the 5th International Conference on Learning
  Representations (ICLR 2017)}, 2017.

\bibitem{kesten1973random}
H.~Kesten.
\newblock {Random difference equations and renewal theory for products of
  random matrices}.
\newblock {\em Acta Mathematica}, 131:207--248, 1973.

\bibitem{kingma2015adam}
D.~P. Kingma and J.~Ba.
\newblock {Adam: A method for stochastic optimization}.
\newblock {\em Proceedings of the 3rd International Conference on Learning
  Representations (ICLR 2015)}, 2015.

\bibitem{kleinberg2018alternative}
R.~Kleinberg, Y.~Li, and Y.~Yuan.
\newblock {An alternative view: When does SGD escape local minima?}
\newblock {\em Proceedings of the 35th International Conference on Machine
  Learning (ICML 2018)}, 2018.

\bibitem{kleywegt2002sample}
A.~J. Kleywegt, A.~Shapiro, and T.~Homem-de Mello.
\newblock {The sample average approximation method for stochastic discrete
  optimization}.
\newblock {\em SIAM Journal on Optimization}, 12(2):479--502, 2002.

\bibitem{knuth1976big}
D.~E. Knuth.
\newblock {Big omicron and big omega and big theta}.
\newblock {\em ACM Sigact News}, 8(2):18--24, 1976.

\bibitem{krizhevsky2009learning}
A.~Krizhevsky and G.~Hinton.
\newblock {Learning multiple layers of features from tiny images}.
\newblock 2009.

\bibitem{kroese2013handbook}
D.~P. Kroese, T.~Taimre, and Z.~I. Botev.
\newblock {\em {Handbook of Monte Carlo methods}}, volume 706.
\newblock John Wiley \& Sons, 2013.

\bibitem{li2019budgeted}
M.~Li, E.~Yumer, and D.~Ramanan.
\newblock {Budgeted Training: Rethinking Deep Neural Network Training Under
  Resource Constraints}.
\newblock {\em Proceedings of the 8th International Conference on Learning
  Representations (ICLR 2020)}, 2020.

\bibitem{liu2019deep}
G.-H. Liu and E.~A. Theodorou.
\newblock {Deep learning theory review: An optimal control and dynamical
  systems perspective}.
\newblock {\em arXiv preprint arXiv:1908.10920}, 2019.

\bibitem{liu2019variance}
L.~Liu, H.~Jiang, P.~He, W.~Chen, X.~Liu, J.~Gao, and J.~Han.
\newblock {On the variance of the adaptive learning rate and beyond}.
\newblock {\em Proceedings of the 8th International Conference on Learning
  Representations (ICLR 2020)}, 2020.

\bibitem{liu2019bad}
S.~Liu, D.~Papailiopoulos, and D.~Achlioptas.
\newblock {Bad global minima exist and SGD can reach them}.
\newblock {\em arXiv preprint arXiv:1906.02613}, 2019.

\bibitem{ma2017power}
S.~Ma, R.~Bassily, and M.~Belkin.
\newblock {The power of interpolation: Understanding the effectiveness of SGD
  in modern over-parametrized learning}.
\newblock {\em Proceedings of the 35th International Conference on Machine
  Learning (ICML 2018)}, 2018.

\bibitem{ma2015complete}
Y.-A. Ma, T.~Chen, and E.~Fox.
\newblock {A complete recipe for stochastic gradient MCMC}.
\newblock In {\em Advances in Neural Information Processing Systems}, pages
  2917--2925, 2015.

\bibitem{mandt2016variational}
S.~Mandt, M.~Hoffman, and D.~Blei.
\newblock {A variational analysis of stochastic gradient algorithms}.
\newblock In {\em Proceedings of the 33rd International Conference on Machine
  Learning (ICML 2016)}, pages 354--363, 2016.

\bibitem{mandt2017stochastic}
S.~Mandt, M.~D. Hoffman, and D.~M. Blei.
\newblock {Stochastic gradient descent as approximate Bayesian inference}.
\newblock {\em Journal of Machine Learning Research}, 18(1):4873--4907, 2017.

\bibitem{martin2017rethinking}
C.~H. Martin and M.~W. Mahoney.
\newblock {Rethinking generalization requires revisiting old ideas: statistical
  mechanics approaches and complex learning behavior}.
\newblock {\em arXiv preprint arXiv:1710.09553}, 2017.

\bibitem{martin2018implicit}
C.~H. Martin and M.~W. Mahoney.
\newblock {Implicit self-regularization in deep neural networks: Evidence from
  random matrix theory and implications for learning}.
\newblock {\em arXiv preprint arXiv:1810.01075}, 2018.

\bibitem{martin2019traditional}
C.~H. Martin and M.~W. Mahoney.
\newblock {Traditional and heavy-tailed self regularization in neural network
  models}.
\newblock {\em Proceedings of the 36th International Conference on Machine
  Learning (ICML 2019)}, 2019.

\bibitem{martin2020heavy}
C.~H. Martin and M.~W. Mahoney.
\newblock {Heavy-tailed Universality predicts trends in test accuracies for
  very large pre-trained deep neural networks}.
\newblock In {\em Proceedings of the 2020 SIAM International Conference on Data
  Mining}, pages 505--513. SIAM, 2020.

\bibitem{martin2020predicting}
C.~H. Martin and M.~W. Mahoney.
\newblock {Predicting trends in the quality of state-of-the-art neural networks
  without access to training or testing data}.
\newblock {\em arXiv preprint arXiv:2002.06716}, 2020.

\bibitem{meyn2012markov}
S.~P. Meyn and R.~L. Tweedie.
\newblock {\em {Markov chains and stochastic stability}}.
\newblock Springer Science \& Business Media, 2012.

\bibitem{jain2019sgd}
D.~Nagaraj, P.~Netrapalli, and P.~Jain.
\newblock {SGD without replacement: Sharper rates for general smooth convex
  functions}.
\newblock {\em Proceedings of the 36th International Conference on Machine
  Learning (ICML 2019)}, 2019.

\bibitem{orvieto2019continuous}
A.~Orvieto and A.~Lucchi.
\newblock {Continuous-time models for stochastic optimization algorithms}.
\newblock In {\em Advances in Neural Information Processing Systems}, pages
  12589--12601, 2019.

\bibitem{oymak2019generalization}
S.~Oymak, Z.~Fabian, M.~Li, and M.~Soltanolkotabi.
\newblock {Generalization, Adaptation and Low-Rank Representation in Neural
  Networks}.
\newblock In {\em 2019 53rd Asilomar Conference on Signals, Systems, and
  Computers}, pages 581--585. IEEE, 2019.

\bibitem{pal2006hysteretic}
K.~F. P{\'a}l.
\newblock {Hysteretic optimization, faster and simpler}.
\newblock {\em Physica A: Statistical Mechanics and its Applications},
  360(2):525--533, 2006.

\bibitem{panigrahi2019non}
A.~Panigrahi, R.~Somani, N.~Goyal, and P.~Netrapalli.
\newblock {Non-gaussianity of stochastic gradient noise}.
\newblock {\em Science meets Engineering of Deep Learning (SEDL) workshop, 33rd
  Conference on Neural Information Processing Systems (NeurIPS 2019)}, 2019.

\bibitem{robbins1951stochastic}
H.~Robbins and S.~Monro.
\newblock {A stochastic approximation method}.
\newblock {\em The Annals of Mathematical Statistics}, pages 400--407, 1951.

\bibitem{roosta2016sub}
F.~Roosta-Khorasani and M.~W. Mahoney.
\newblock {Sub-sampled Newton methods II: Local convergence rates}.
\newblock {\em arXiv preprint arXiv:1601.04738}, 2016.

\bibitem{rubin2014mapping}
K.~J. Rubin, G.~Pruessner, and G.~A. Pavliotis.
\newblock {Mapping multiplicative to additive noise}.
\newblock {\em Journal of Physics A: Mathematical and Theoretical},
  47(19):195001, 2014.

\bibitem{sandler2018mobilenetv2}
M.~Sandler, A.~Howard, M.~Zhu, A.~Zhmoginov, and L.-C. Chen.
\newblock {MobileNetv2: Inverted residuals and linear bottlenecks}.
\newblock In {\em Proceedings of the IEEE Conference on Computer Vision and
  Pattern Recognition}, pages 4510--4520, 2018.

\bibitem{simsekli2019tail}
U.~{\c{S}}im{\c{s}}ekli, L.~Sagun, and M.~G{\"u}rb{\"u}zbalaban.
\newblock {A tail-index analysis of stochastic gradient noise in deep neural
  networks}.
\newblock {\em Proceedings of the 36th International Conference on Machine
  Learning (ICML 2019)}, 2019.

\bibitem{simsekli2020fractional}
U.~{\c{S}}im{\c{s}}ekli, L.~Zhu, Y.~W. Teh, and M.~G{\"u}rb{\"u}zbalaban.
\newblock {Fractional Underdamped Langevin Dynamics: Retargeting SGD with
  Momentum under Heavy-Tailed Gradient Noise}.
\newblock {\em arXiv preprint arXiv:2002.05685}, 2020.

\bibitem{smith2018disciplined}
L.~N. Smith.
\newblock {A disciplined approach to neural network hyper-parameters: Part
  1--learning rate, batch size, momentum, and weight decay}.
\newblock {\em US Naval Research Laboratory Technical Report 5510-026}, 2018.

\bibitem{smith2017don}
S.~L. Smith, P.-J. Kindermans, C.~Ying, and Q.~V. Le.
\newblock {Don't decay the learning rate, increase the batch size}.
\newblock {\em Proceedings of the 6th International Conference on Learning
  Representations (ICLR 2018)}, 2018.

\bibitem{sornette1997convergent}
D.~Sornette and R.~Cont.
\newblock {Convergent multiplicative processes repelled from zero: power laws
  and truncated power laws}.
\newblock {\em Journal de Physique I}, 7(3):431--444, 1997.

\bibitem{srivastava2014dropout}
N.~Srivastava, G.~Hinton, A.~Krizhevsky, I.~Sutskever, and R.~Salakhutdinov.
\newblock {Dropout: a simple way to prevent neural networks from overfitting}.
\newblock {\em Journal of Machine Learning Research}, 15(1):1929--1958, 2014.

\bibitem{volpe2016effective}
G.~Volpe and J.~Wehr.
\newblock {Effective drifts in dynamical systems with multiplicative noise: a
  review of recent progress}.
\newblock {\em Reports on Progress in Physics}, 79(5):053901, 2016.

\bibitem{perez2017effectiveness}
J.~Wang and L.~Perez.
\newblock {The effectiveness of data augmentation in image classification using
  deep learning}.
\newblock {\em Convolutional Neural Networks Vis. Recognit}, page~11, 2017.

\bibitem{ward2018adagrad}
R.~Ward, X.~Wu, and L.~Bottou.
\newblock {Adagrad stepsizes: Sharp convergence over nonconvex landscapes, from
  any initialization}.
\newblock {\em Proceedings of the 36th International Conference on Machine
  Learning (ICML 2019)}, 2019.

\bibitem{wenzel2020good}
F.~Wenzel, K.~Roth, B.~S. Veeling, J.~{\'S}wi{\k{a}}tkowski, L.~Tran, S.~Mandt,
  J.~Snoek, T.~Salimans, R.~Jenatton, and S.~Nowozin.
\newblock {How good is the bayes posterior in deep neural networks really?}
\newblock {\em arXiv preprint arXiv:2002.02405}, 2020.

\bibitem{wu2020noisy}
J.~Wu, W.~Hu, H.~Xiong, J.~Huan, V.~Braverman, and Z.~Zhu.
\newblock {On the Noisy Gradient Descent that Generalizes as SGD}.
\newblock {\em arXiv preprint arXiv:1906.07405}, 2020.

\bibitem{xing2018walk}
C.~Xing, D.~Arpit, C.~Tsirigotis, and Y.~Bengio.
\newblock {A Walk with SGD}.
\newblock {\em arXiv preprint arXiv:1802.08770}, 2018.

\bibitem{yao2018hessian}
Z.~Yao, A.~Gholami, Q.~Lei, K.~Keutzer, and M.~W. Mahoney.
\newblock {Hessian-based analysis of large batch training and robustness to
  adversaries}.
\newblock In {\em Advances in Neural Information Processing Systems}, pages
  4949--4959, 2018.

\bibitem{zhang2019algorithmic}
G.~Zhang, L.~Li, Z.~Nado, J.~Martens, S.~Sachdeva, G.~Dahl, C.~Shallue, and
  R.~B. Grosse.
\newblock {Which algorithmic choices matter at which batch sizes? insights from
  a noisy quadratic model}.
\newblock In {\em Advances in Neural Information Processing Systems}, pages
  8194--8205, 2019.

\bibitem{zhang2018removing}
Z.~Zhang, Y.~Zhang, and Z.~Li.
\newblock {Removing the feature correlation effect of multiplicative noise}.
\newblock In {\em Advances in Neural Information Processing Systems}, pages
  627--636, 2018.

\bibitem{zhou2018convergence}
D.~Zhou, Y.~Tang, Z.~Yang, Y.~Cao, and Q.~Gu.
\newblock {On the convergence of adaptive gradient methods for nonconvex
  optimization}.
\newblock {\em arXiv preprint arXiv:1808.05671}, 2018.

\end{thebibliography}

\appendix

\section{Random linear recurrence relations}
\label{sec:LinearRecurrence}

Here, we shall discuss existing theory concerning the random linear recurrence relation $W_{k+1} = A_k W_k + B_k$ that arises in (\ref{eq:LinearRecurrence}).
Because $(A_k,B_k)$ for each $k=0,1,2,\dots$ is independent and identically distributed, we let $(A,B) = (A_0, B_0)$, noting that $(A,B) \eqd (A_k,B_k)$ for all $k$. 
First, we state conditions under which (\ref{eq:LinearRecurrence}) yields an ergodic Markov chain. The following lemma combines \cite[Theorem 4.1.4 and Proposition 4.2.1]{buraczewski2016stochastic} and implies Lemma~\ref{lem:LinearRecur}.
\begin{lemma}
Suppose that $A$ and $B$ are non-deterministic and both $\log^+\|A\|$ and $\log^+\|B\|$ are integrable. 
Then if $\mathbb{E}\log\|A\| < 0$, the Markov chain (\ref{eq:LinearRecurrence}) has a unique stationary distribution. If also either $A$ or $B$ is non-atomic, then the Markov chain (\ref{eq:LinearRecurrence}) is ergodic.
\end{lemma}

The intuition behind the presence of heavy-tailed behaviour is easily derived from the Breiman lemma concerning regularly varying random vectors. Recall that a random vector $X$ is regularly varying if there exists a measure $\mu_X$ with zero mass at infinity such that
\begin{equation}
\label{eq:RegularlyVarying}
\lim_{x \to \infty} \frac{\mathbb{P}(x^{-1} X \in B)}{\mathbb{P}(\|X\| > x)} = \mu_X(B),\quad\mbox{for any set $B$ satisfying $\mu(\partial B) = 0$.}
\end{equation}
By Karamata's characterization theorem \cite[Theorem 1.4.1]{bingham1989regular}, for any regularly varying random vector $X$, there exists an $\alpha > 0$ such that $x^{\alpha} \mathbb{P}(x^{-1} X \in \cdot)$ converges as $x \to \infty$ to a non-null measure. In particular, $\|X\|$ and $|\langle u, X\rangle|$ for every $u \in \mathbb{R}^d$ obey a power law with tail exponent $\alpha$ subject to slowly varying functions\footnote{Recall that a function $f$ is slowly varying if $f(t x) / f(x) \to 1$ as $\|x\| \to \infty$, for any $t > 0$.} $L$, $L_u$:
\begin{equation}
\label{eq:PowerLaw}
\mathbb{P}(\|X\| > x) \sim L(x) x^{-\alpha},\qquad \mathbb{P}(|\langle u, X\rangle| > x) \sim L_u(x) x^{-\alpha}.
\end{equation}
A random vector $X$ satisfying (\ref{eq:RegularlyVarying}) and (\ref{eq:PowerLaw}) is said to be \emph{regularly varying with index $\alpha$} (abbreviated $\alpha$-RV). The following is \cite[Lemma C.3.1]{buraczewski2016stochastic}.
\begin{lemma}[\textsc{Breiman's lemma}]
Let $X$ be an $\alpha$-RV random vector, $A$ a random matrix such that\footnote{For example, $A$ could be $\beta$-RV with $\beta > \alpha$.} $\mathbb{E}\|A\|^{\alpha + \epsilon} < +\infty$, and $B$ a random vector such that $\mathbb{P}(\|B\| > x) = o(\mathbb{P}(\|X\| > x))$ as $x \to \infty$. Then $AX + B$ is $\alpha$-RV.
\end{lemma}
In other words, the index of regular variation is preserved under random linear operations, and so regularly varying random vectors are distributional fixed points of random linear recurrence relations. Conditions for the converse statement are well-known in the literature \cite{buraczewski2016stochastic}. Here, we provide brief expositions of the three primary regimes dictating the tails of any stationary distribution of (\ref{eq:LinearRecurrence}). It is worth noting that other corner cases do exist, including super-heavy tails (see \cite[Section 5.5]{buraczewski2016stochastic} for example), but are outside the scope of this paper.

\subsection{The Goldie--Gr\"{u}bel (light-tailed) regime}
\label{sec:GoldieGrubel}

To start, consider the case where neither $A$ nor $B$ are heavy-tailed and the stochastic optimization dynamics are such that $W_{\infty}$ is light-tailed. 
In particular, assume that all moments of $B$ are finite. By applying the triangle inequality to (\ref{eq:LinearRecurrence}), one immediately finds
\[
\|W_{k+1}\| \leq \|A_k\|\|W_k\| + \|B_k\|,\quad \mbox{and}\quad \|W_{k+1}\|_{\alpha} \leq \|A\|_{\alpha} \|W_k\|_{\alpha} + \|B\|_{\alpha}.
\]
Therefore, if $\|A\| \leq 1$ almost surely and $\mathbb{P}(\|A\| < 1) > 0$, then for any $\alpha \geq 1$, $\|A\|_{\alpha} < 1$ and so $\|W_k\|_{\alpha}$ is bounded in $k$. The Markov chain (\ref{eq:LinearRecurrence}) is clearly ergodic, and the existence of all moments suggests that the limiting distribution $W_{\infty}$ of $W_k$ cannot satisfy a power law. With significant effort, one can show that more is true: Goldie and Gr\"{u}bel proved in \cite[Theorem 2.1]{goldie1996perpetuities} that if $B$ is also light-tailed, then $W_{\infty}$ is \textbf{light-tailed}. To our knowledge, this is the only setting where one can prove that the Markov chain (\ref{eq:LinearRecurrence}) possesses a light-tailed limiting distribution, and it requires contraction (and therefore, consistent linear convergence) at every step with probability one. In the stochastic optimization setting, the Goldie--Gr\"{u}bel regime coincides with optimizers that have purely exploitative (no explorative) behaviour. Should the chain fail to contract even once, we move outside of this regime and enter the territory of heavy-tailed stationary distributions. 

\subsection{The Kesten--Goldie (heavy-tailed due to intrinsic factors) regime}
\label{sec:KestenGoldie}

Next, consider the case where neither $A$ nor $B$ are heavy-tailed, but the stochastic optimization dynamics are such that $W_{\infty}$ is heavy-tailed. 
To consider a \emph{lower bound}, recall that the smallest singular value of $A$, $\sigma_{\min}(A)$, satisfies $\sigma_{\min}(A) = \inf_{\|w\| = 1} \|A w\|$. Therefore, once again from (\ref{eq:LinearRecurrence}),
\[
\|W_{k+1}\| \geq \sigma_{\min}(A_k)\|W_k\| - \|B_k\|,\quad \mbox{and}\quad \|W_{k+1}\|_{\alpha} \geq \|\sigma_{\min}(A)\|_{\alpha} \|W_k\|_{\alpha} - \|B\|_{\alpha}.
\]
Assuming that the Markov chain (\ref{eq:LinearRecurrence}) is ergodic with limiting distribution $W_\infty$, by the $f$-norm ergodic theorem \cite[Theorem 14.0.1]{meyn2012markov}, $\|W_{\infty}\|_\alpha$ is finite if and only if $\|W_k\|_\alpha$ is bounded in $k$ for any initial $W_0$. However, if $\mathbb{P}(\sigma_{\min}(A) > 1) > 0$, then there exists some $\alpha > 1$ such that $\|\sigma_{\min}(A)\|_{\alpha} > 1$. If $\|B\|_{\alpha}$ is finite, then $\|W_k\|_{\alpha}$ is unbounded when $\|W_0\|_{\alpha}$ is sufficiently large, implying that $W_{\infty}$ is \textbf{heavy-tailed}. 

This suggests that the tails of the distribution of $\|W_{\infty}\|$ are at least as heavy as a power law. To show they are dictated \emph{precisely} by a power law, that is, $W_{\infty}$ is $\alpha$-RV for some $\alpha > 0$, is more challenging. The following theorem is a direct corollary of the Kesten's celebrated theorem \cite[Theorem 6]{kesten1973random}, and Goldie's generalizations thereof in \cite{goldie1991implicit}.

~
\begin{theorem}[\textsc{Kesten--Goldie theorem}]
\label{thm:Kesten}
Assume the following:
\begin{itemize}[leftmargin=*]
    \item The Markov chain (\ref{eq:LinearRecurrence}) is ergodic with $W_{\infty} = \lim_{k\to\infty} W_k$ (in distribution).
    \item The distribution of $X$ has absolutely continuous component with respect to Lebesgue density that has support containing the zero matrix, and $Y$ is non-zero with positive probability.
    \item There exists $s > 0$ such that $\mathbb{E}\sigma_{\min}(A)^s = 1$.
    \item $A$ is almost surely invertible and $\mathbb{E}[\|A\|^s \log^+\|A\|] + \mathbb{E}[\|A\|^s \log^+\|A^{-1}\|] < \infty$.
    \item $\mathbb{E}\|B\|^s < \infty$.
\end{itemize}
Then $W_{\infty}$ is $\alpha$-RV for some $0 < \alpha \leq s$.
\end{theorem}

\subsection{The Grincevi\v{c}ius--Grey (heavy-tailed due to extrinsic factors) regime}
\label{sec:GrinceviciusGrey}

Finally, consider the case where $B$ is heavy-tailed, in particular, that $B$ is $\beta$-RV. If $\|A\|_{\beta} < 1$, then the arguments seen in the Kesten--Goldie regime can no longer hold, since $\|B\|_{\alpha}$ would be infinite for any $\alpha$ such that $\|\sigma_{\min}(A)\|_{\alpha} = 1$. Instead, by \cite[Theorem 4.4.24]{buraczewski2016stochastic}, a limiting distribution of (\ref{eq:LinearRecurrence}) is necessarily $\beta$-RV, provided that $\|A\|_{\beta + \delta}$ is finite for some $\delta > 0$. This was proved in the univariate case by Grincevi\v{c}ius \cite[Theorem 1]{grincevicius1975one}, later updated by Grey \cite{grey1994regular} to include the converse result: if the limiting distribution of (\ref{eq:LinearRecurrence}) is $\beta$-RV and $\|A\|_{\beta} < 1$, then $B$ is $\beta$-RV. 

Contrary to the Kesten--Goldie regime, here neither $A$ nor the recursion itself play much of a role. The optimization procedure itself is fairly irrelevant: from Breiman's lemma, the distribution of $W_k$ is heavy-tailed after only the first iteration, and the tail exponent remains constant, i.e., $W_{\infty}$ is heavy-tailed. Therefore, in the Grincevi\v{c}ius--Grey regime, the dynamics of the stochastic optimization are dominated by \textbf{extrinsic~factors}. 

\section{Numerical examinations of heavy tails}
\label{sec:Numerics2}

Power laws are notoriously treacherous to investigate empirically \cite{clauset2009power}, especially in higher dimensions \cite{panigrahi2019non}, and this plays a significant role in our focus on establishing mathematical theory. Nevertheless, 
due to the mystique surrounding heavy tails and our discussion in \S\ref{sec:General} concerning the impact of various factors on the tail exponent being predominantly informal, we also recognize the value of empirical confirmation. Here, we shall conduct a few numerical examinations to complement our main discussion. For further empirical analyses concerning non-Gaussian fluctuations in stochastic optimization, we refer to \cite{simsekli2019tail,panigrahi2019non}.

\begin{figure}
    \centering
    \includegraphics[width=0.48\textwidth]{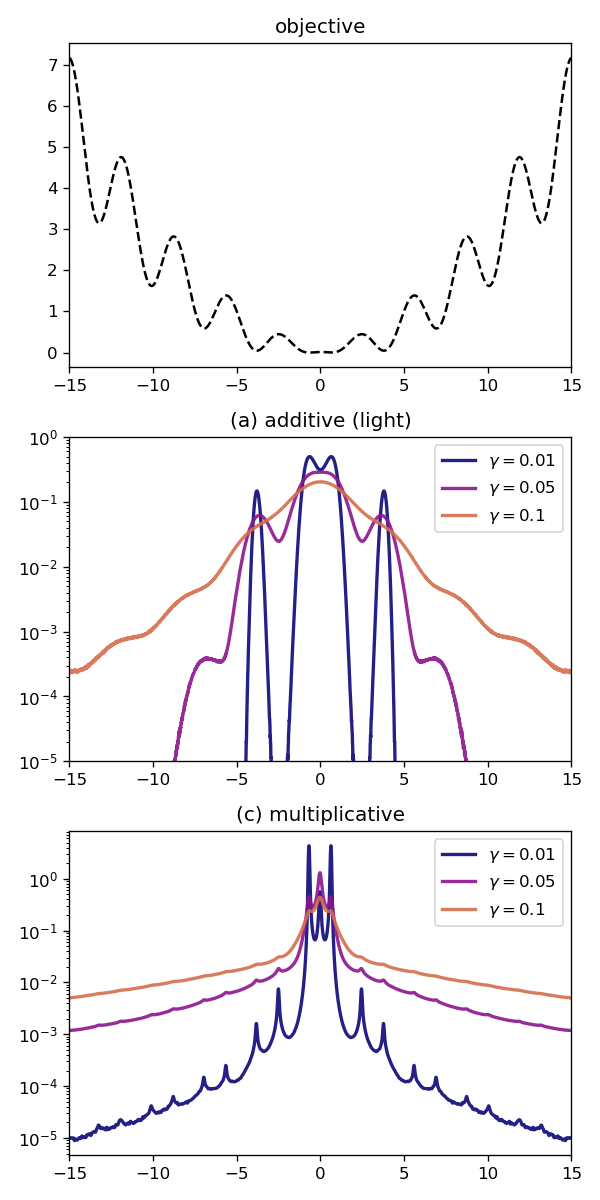}
    \caption{Estimated stationary distributions for optimizers (a) and (c) applied to a non-convex objective $f$ with derivative $f'(x) = x(1-4\cos(2 x))$, over varying step sizes $\gamma$.}
    \label{fig:HeavyTail1D}
\end{figure}

As a quick illustration, in Figure \ref{fig:HeavyTail1D}, we contrast tail behaviour in the stationary distributions of the Markov chains induced by optimizers (a) (additive) and (c) (multiplicative) introduced in \S\ref{sec:Numerics}. Three different step sizes are used, with constant $\sigma = 10$. To exacerbate multimodality in the stationary distribution, we consider an objective $f$ with derivative $f'(x) = x(1-4\cos(2 x))$, visualized in the upper part of Figure \ref{fig:HeavyTail1D}. Accurately visualizing the stationary distribution, especially its tails, is challenging: to do so, we apply a low bandwidth kernel density estimate to $10^9$ steps. As expected, multiplicative noise exhibits slowly decaying heavy tails in contrast to the rapidly decaying Gaussian tails seen with additive light noise. Furthermore, the heaviness of the tails increases with the step size.

To estimate the power law from empirically obtained data, in the sequel, we shall make use of the \texttt{powerlaw} Python package \cite{alstott2014powerlaw}, which applies a combination of maximum likelihood estimation and Kolmogorov-Smirnov techniques (see \cite{clauset2009power}) to fit a Pareto distribution to data. Recall that a Pareto distribution has density $p(t) = \beta t_{\min}^\beta t^{-\beta}$ for $t \geq t_{\min}$, where $t_{\min}$ is the scale parameter (that is, where the power law in the tail begins), and $\beta$ is the tail exponent in the density. Note that this $\beta$ is related to our definition of the tail exponent $\alpha$ by $\alpha = \beta - 1$. Unbiased estimates of this tail exponent $\alpha$ obtained from the \texttt{powerlaw} package will be denoted by $\hat{\alpha}$. 

\subsection{The linear case with SGD}
Let us reconsider the simple case discussed in \S\ref{sec:LinearOptim} and illustrate power laws arising from SGD on ridge regression. As a particularly simple illustration, first consider the one-dimensional case of (\ref{eq:LinearRecurrence}) with $n = 1$, $\gamma = \frac12$, $\lambda = 0$, and standard normal synthetic data. The resulting Markov chain is 
\begin{equation}
\label{eq:LinearToy1D}
W_{k+1} = (1 - \tfrac12 X_k^2)W_k - \tfrac12 X_k Y_k,
\end{equation}
where $X_k,Y_k \iidsim \mathcal{N}(0,1)$.
Starting from $W_0 = 0$, Figure \ref{fig:LinearToy1D} shows a trace plot of $10^7$ iterations of (\ref{eq:LinearToy1D}). One can observe the sporadic spikes that are indicative of heavy-tailed fluctuations. Also in Figure \ref{fig:LinearToy1D} is an approximation of the probability density of magnitudes of the iterations. Both exponential and log-normal fits are obtained via maximum likelihood estimation and compared with the power law predicted from Theorem \ref{thm:Kesten} ($\alpha \approx 2.90$). 
Visually, the power law certainly provides the best fit. Using the Python package \texttt{powerlaw}, a Pareto-distribution was fitted to the iterations. The theoretical tail exponent falls within the 95\% confidence interval for the estimated tail exponent: $\hat{\alpha} = 2.95 \pm 0.06$. However, even for this simple case where the stationary distribution is known to exhibit a power law and a significant number of samples are available, a likelihood ratio test was found incapable of refuting a Kolmogorov-Smirnov lognormal fit to the tail. 

\begin{figure}
    \centering
    \includegraphics[width=\textwidth]{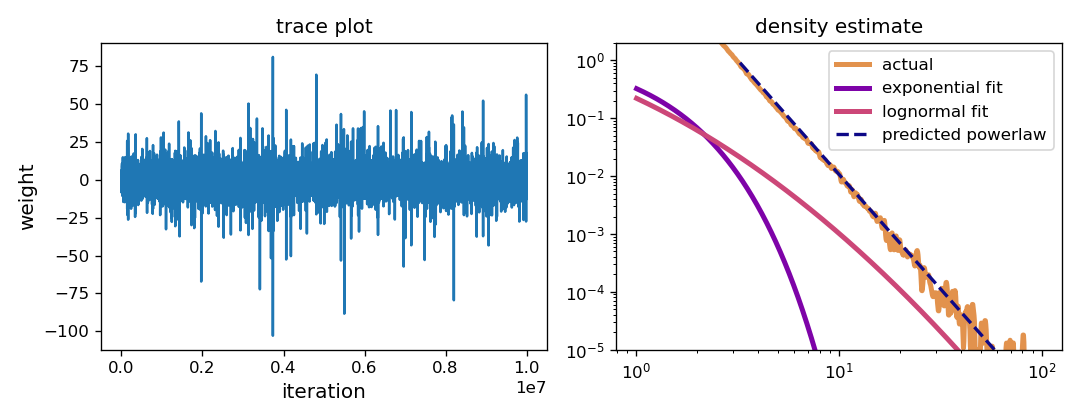}
    \caption{(Left) A trace plot of $10^7$ iterations of (\ref{eq:LinearToy1D}) and (right) a corresponding probability density estimate of their absolute values, with exponential, log-normal fits, and the power law predicted by Theorem \ref{thm:Kesten}.}
    \label{fig:LinearToy1D}
\end{figure}

As the dimension increases, the upper bound on the power law from Theorem \ref{thm:Kesten} becomes increasingly less tight. To see this, we conduct least-squares linear regression to the Wine Quality dataset \cite{cortez2009modeling} (12 attributes; 4898 instances) using vanilla SGD with step size $\gamma = 0.3$, $L^2$ regularization parameter $\lambda = 4$, and minibatch size $n=1$. These parameters are so chosen to ensure that the resulting sequence of iterates satisfying (\ref{eq:LinearRecurrence}) is just barely ergodic, and exhibits a profoundly heavy tail. Starting from standard normal $W_0$, Figure \ref{fig:LinearBig} shows a trace plot of $2.5$ million iterations, together with an approximation of the probability density of magnitudes of the iterations. A Pareto-distribution fit obtained using the \texttt{powerlaw} package is also drawn and can be seen to be an excellent match to the data; the corresponding estimated tail exponent is $\hat{\alpha} \approx 0.446 \pm 0.008$ to 95\% confidence. 
However, applying Theorem \ref{thm:Kesten} to the chain formed from every $12$ iterations reveals a much larger upper bound on the tail exponent: $\alpha \leq 22$.

\begin{figure}
    \centering
    \includegraphics[width=\textwidth]{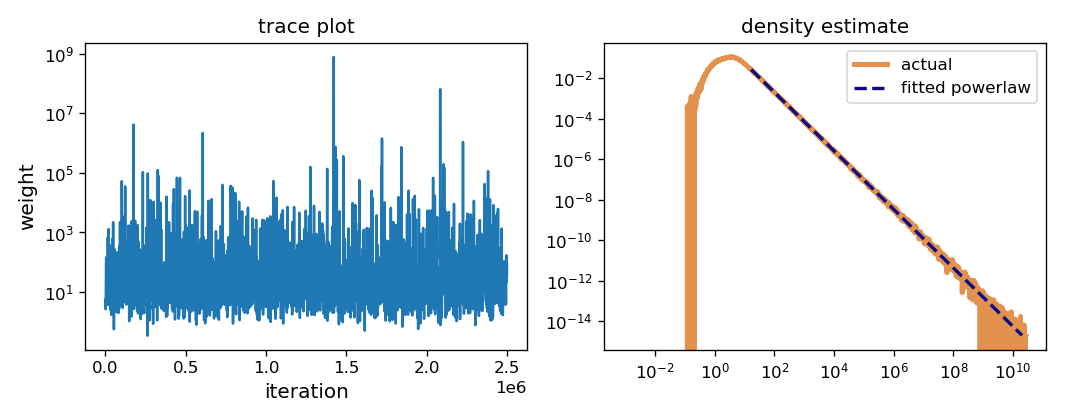}
    \caption{(Left) A trace plot of $2,500,000$ iterations of (\ref{eq:LinearRecurrence}) for the Wine Quality dataset and (right) a corresponding probability density estimate of their norms, with fitted Pareto distribution.}
    \label{fig:LinearBig}
\end{figure}

\subsection{Factors influencing tail exponents}

To help support the claims in \S\ref{sec:General} concerning the influence of factors on tail exponents, we conducted least-squares regression to the Wine Quality dataset \cite{cortez2009modeling} (12 attributes; 4898 instances) using a two-layer neural network with four hidden units and ReLU activation function. 
Our baseline stochastic optimizer is vanilla SGD with a constant step size of $\gamma = 0.025$, minibatch size $n=1$, and $L^2$ regularization parameter $\lambda = 10^{-4}$. The effect of changing each of these hyperparameters individually was examined. Three other factors were also considered: (i) the effect of smoothing input data by adding Gaussian noise with standard deviation $\epsilon$; (ii) the effect of adding momentum (and increasing this hyperparameter); and (iii) changing the optimizer between SGD, Adagrad, Adam, and subsampled Newton (SSN). 

\begin{table}
    \centering
\begin{tabular}{| c | c | c | c | c | c | c | c |}
\hline
\rowcolor{GrayH} \multicolumn{2}{|c|}{\textbf{step size}} & \multicolumn{2}{|c|}{\textbf{minibatch size}} & \multicolumn{2}{|c|}{$\boldsymbol{L^2}$\textbf{ regularization}} & \multicolumn{2}{|c|}{\textbf{data smoothing}} \\
\hline
\rowcolor{Gray} $\gamma$ & $\hat{\alpha}$ & $n$ & $\hat{\alpha}$ & $\lambda$ & $\hat{\alpha}$ & $\epsilon$ & $\hat{\alpha}$ \\
\hline
0.001 & 4.12 $\pm$ 0.04 & 10 & 5.99 $\pm$ 0.05 & $10^{-4}$ & 2.97 $\pm$ 0.03 & 0 & 2.97 $\pm$ 0.03\\
0.005 & 3.70 $\pm$ 0.02 & 5 & 4.98 $\pm$ 0.07 & 0.01 & 3.02 $\pm$ 0.02 & 0.1 & 2.96 $\pm$ 0.05\\
0.01 & 3.71 $\pm$ 0.04 & 2 & 3.62 $\pm$ 0.03 & 0.1 & 2.77 $\pm$ 0.01 & 0.5 & 3.05 $\pm$ 0.02 \\
0.025 & 2.97 $\pm$ 0.03 & 1 & 2.97 $\pm$ 0.03 & 0.2 & 2.55 $\pm$ 0.01 & 1 & 2.36 $\pm$ 0.13  \\
\hline
\end{tabular}
\begin{tabular}{| c | c | c | c |}
\hline
\rowcolor{GrayH} \multicolumn{2}{|c|}{\textbf{momentum}} & \multicolumn{2}{|c|}{\textbf{optimizer}} \\
\hline
\rowcolor{Gray} $\eta$ & $\hat{\alpha}$ &  & $\hat{\alpha}$ \\
\hline
0.5 & 4.99 $\pm$ 0.03 & Adagrad & 3.2 $\pm$ 0.1 \\
0.25 & 2.48 $\pm$ 0.02 & Adam & 2.119 $\pm$ 0.005 \\
0.1 & 2.84 $\pm$ 0.02 & SGD & 2.93 $\pm$ 0.03 \\
0 & 2.93 $\pm$ 0.03 & SSN & 0.79 $\pm$ 0.04 \\
\hline
\end{tabular}
    \caption{Estimated tail exponents for the distributions of norms of steps in roughly 2.5 million stochastic optimization iterations, varying only one hyperparameter from a baseline step size $\gamma = 0.025$, minibatch size $n = 1$, $L^2$ regularization parameter $\lambda = 10^{-4}$, no Gaussian perturbations to input data ($\epsilon = 0$), using SGD with momentum parameter $\eta = 0$. The Adagrad, Adam, and SSN optimizers are also considered, using the same baseline hyperparameters (and $\beta_1 = 0.9$, $\beta_2 = 0.999$ for Adam). }
    \label{tab:FactorsTailExp}
\end{table}
\begin{figure}
    \centering
    \includegraphics[width=\textwidth]{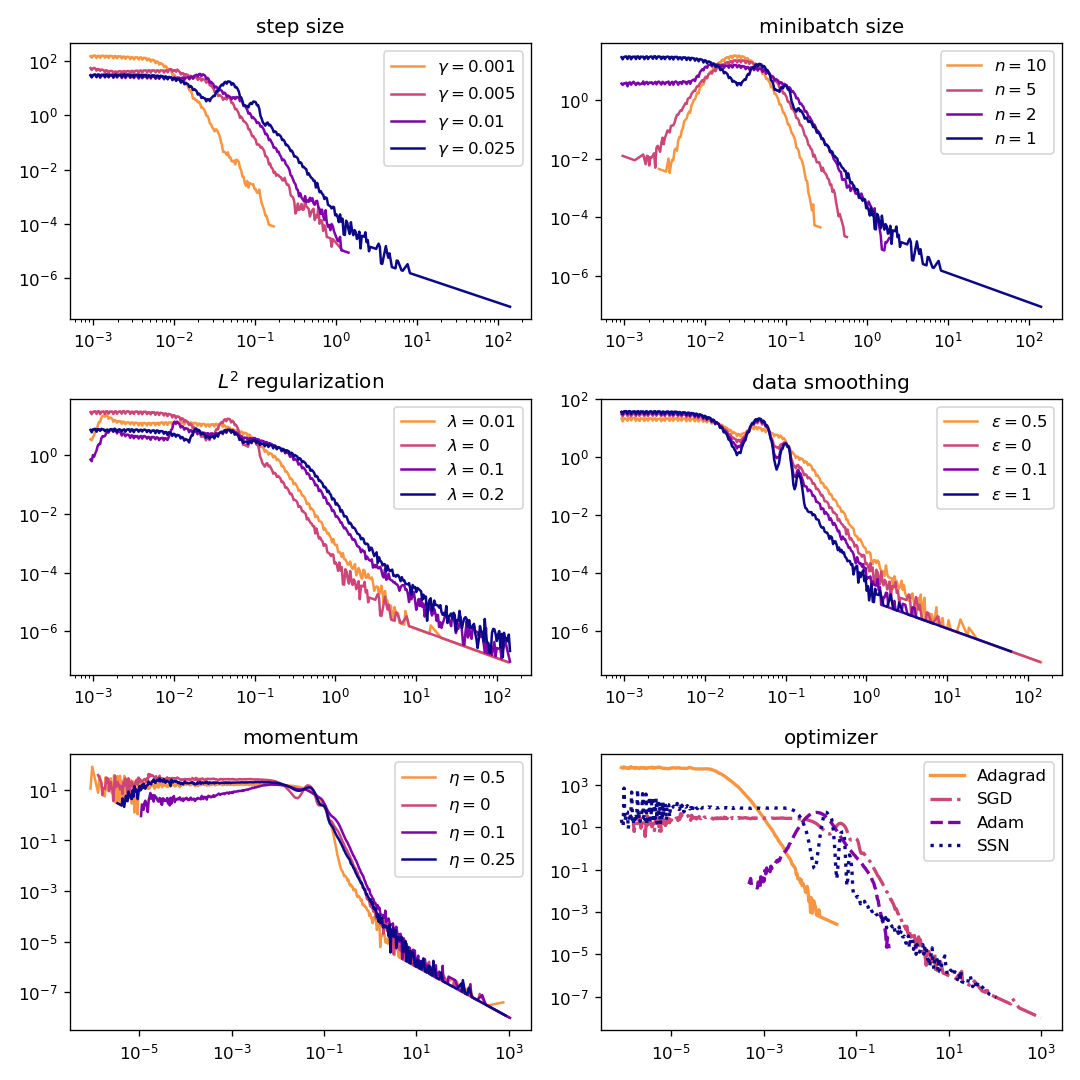}
    \caption{Density estimates of the distributions of norms of steps in roughly 2.5 million stochastic optimization iterations, varying only one hyperparameter from a baseline step size $\gamma = 0.025$, minibatch size $n = 1$, $L^2$ regularization parameter $\lambda = 10^{-4}$, no Gaussian perturbations to input data ($\epsilon = 0$), using SGD with momentum parameter $\eta = 0$. The Adagrad, Adam, and SSN optimizers are also considered, using the same baseline hyperparameters (and $\beta_1 = 0.9$, $\beta_2 = 0.999$ for Adam). Darker colours indicate smaller estimated tail exponents and heavier tails (from Table \ref{tab:FactorsTailExp}).
    }
    \label{fig:FactorsMain}
\end{figure}

In each case, we ran the stochastic optimizer for $500n$ epochs (roughly 2.5 million iterations). Instead of directly measuring norms of the weights, we prefer to look at norms of the steps $W_{k+1} - W_k$. There are two reasons for this: (1) unlike steps, the mode of the stationary distribution of the norms of the weights will not be close to zero, incurring a significant challenge to the estimation of tail exponents; and (2) if steps at stationarity are heavy-tailed in the sense of having infinite $\alpha$th moment, then the stationary distribution of the norms of the weights will have infinite $\alpha$th moment also. This is due to the triangle inequality: assuming $\{W_k\}_{k=1}^{\infty}$ is ergodic with $W_k \cd W_\infty$, $\|W_{\infty}\|_\alpha \geq \frac12 \limsup_{k\to\infty} \|W_{k+1} - W_k\|_\alpha$. 
Density estimates for the steps of each run, varying each factor individually, are displayed in Figure \ref{fig:FactorsMain}. Using the \texttt{powerlaw} package, tail exponents were estimated in each case, and are presented in Table \ref{tab:FactorsTailExp} as 95\% confidence intervals. As expected, both increasing step size and decreasing minibatch size can be seen to decrease tail exponents, resulting in heavier tails. Unfortunately, the situation is not as clear for the other factors; from Figure \ref{fig:FactorsMain}, we can see that this is possibly due in part to the unique shapes of the distributions, preventing effective estimates of the scale parameter, upon which the tail exponent (and its confidence intervals) are dependent. Nevertheless, there are a few comments that we can make. Firstly, the inclusion of momentum does not seem to prohibit heavy-tailed behaviour, even though the theory breaks down in these cases. On the other hand, as can be seen in Figure \ref{fig:FactorsMain}, Adam appears to exhibit very light tails compared to other optimizers. Adagrad exhibits heavy-tailed behaviour despite taking smaller steps on average. SSN shows the strongest heavy-tailed behaviour among all the stochastic optimizers considered. Increasing $L^2$ regularization does increase variance of the steps, but does not appear to make a significant difference to the tails in this test case. Similarly, the effect of adding noise to the data is unclear, although our claim that increasing dispersion of the data (which the addition of large amounts of noise would certainly do) results in heavier-tailed behaviour, is supported by the $\epsilon = 1$ case. 

\subsection{Non-Gaussian stationary behaviour in MobileNetV2}
To illustrate that heavy-tailed phenomena is not limited to small models, we highlight the presence of heavy tails in MobileNetV2 \cite{sandler2018mobilenetv2} trained to the CIFAR10 dataset \cite{krizhevsky2009learning} upscaled using bicubic filtering by a factor of 7. Applying transfer learning, we begin with pretrained weights for ImageNet, replacing the final fully-connected layer with a 10-class classifier and fine-tune the model, fixing all weights except the final layer. Norms of the steps $W_{k+1}-W_k$ for three epochs of vanilla SGD with minibatch size $n=1$, $L^2$ regularization parameter $\lambda = 5 \times 10^{-4}$, and a step size of $\gamma = 0.1$ (far larger than common choices for training MobileNetV2, instead chosen to exacerbate heavy tails for the purpose of illustration) were recorded. These are displayed in Figure \ref{fig:MobileNetV2}: on the left side is a density estimate obtained via pseudo-Gaussian smoothing with large bandwidth; on the right side is the histogram of steps on log-log scale together with the densities of corresponding exponential and log-normal fits, obtained via maximum likelihood estimation. A maximum likelihood estimate of the tail exponent $\alpha$ for a power law distributional fit reveals $\hat{\alpha} = 1.52 \pm 0.01$ to 95\% confidence. Evidently, steps are not light-tailed: a likelihood ratio test strongly supports rejection of the light-tailed (exponential) null hypothesis ($p < 10^{-8}$). Once again, we are unable to refute the hypothesis that the data comes from a log-normal distribution.
\begin{figure}
    \centering
    \includegraphics[width=\textwidth]{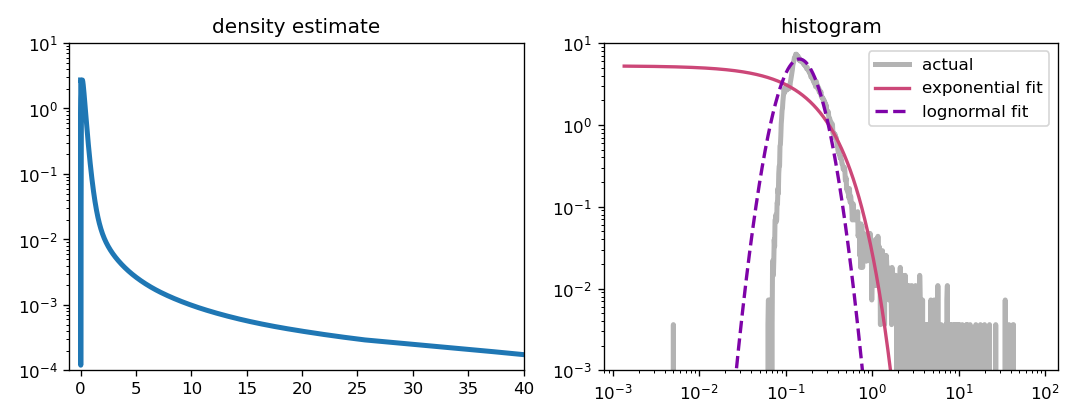}
    \caption{(Left) density estimate of norms of steps of vanilla SGD with MobileNetV2 on CIFAR10, and (right) histogram of steps with exponential and log-normal maximum likelihood fits}
    \label{fig:MobileNetV2}
\end{figure}

\section{Proofs}
\label{app:Proofs}

\begin{proof}[Proof of Lemma~\ref{lem:LinearHeavy}]
Observe that $\bar{X}_k = n^{-1}\sum_{i=1}^n X_{ik}X_{ik}^\top$ has full support on the space of square $d \times d$ matrices when $n \geq d$, and full support on rank-$n$ $d \times d$ matrices otherwise. In the former (over)determined case, both the smallest and largest singular values of $\bar{X}_k$ have full support on $[0,\infty)$, and hence the stationary distribution of (\ref{eq:LinearRecurrence}) decays as a power law. The latter underdetermined case is not so straightforward: while $\sigma_{\max}(A)$ still has full support on $[0,\infty)$, $\sigma_{\min}(A) \leq |1 - \lambda|$ almost surely. Instead, observe that the Markov chain that arises from taking $d$ steps is also a random linear recurrence relation with the same stationary distribution as (\ref{eq:LinearRecurrence}):
\[
W_{k+d} = A^{(d)}_k W_k + B^{(d)}_k,\quad \mbox{where}\quad A^{(d)}_k = A_{k+d}\cdots A_k,\quad B^{(d)}_k = \sum_{l=0}^{d-1}A_{k+d-1}\cdots A_{k+l+1} B_{k+l}.
\]
One may verify that $A^{(d)}$ has full support on the space of \emph{all} square $d \times d$ matrices, and hence, (\ref{eq:LinearRecurrence}) will also exhibit a stationary distribution with heavy tails.
\end{proof}

For the general case, our strategy is to prove that, for some $\alpha>0$, $\mathbb{E}|f(W_{k})|^{\alpha}$ diverges as $k\to\infty$. If we assume that $W$ is ergodic with limiting distribution $W_{\infty}$, then the $f$-norm ergodic theorem \cite[Theorem 14.0.1]{meyn2012markov} implies that $\mathbb{E}|f(W_{k})|^{\alpha}$ converges if and only if $\mathbb{E}|f(W_{\infty})|^{\alpha}$ is finite, hence the divergence of $\mathbb{E}|f(W_{k})|^{\alpha}$ implies $|f(W_{\infty})|$ has infinite $\alpha$-th moment. In particular, here is the proof of Lemma~\ref{lem:Abstract}. 

\begin{proof}[Proof of Lemma~\ref{lem:Abstract}]
It suffices to show that $\|f(W_\infty)\|_{\beta} =+\infty$ for any $\beta > \alpha$, where
\[
\alpha \coloneqq \inf_{\epsilon > 0}\frac{1}{\log(1+\epsilon)}\left|\log\inf_{w\in S}\mathbb{P}\left(\frac{|f(\Psi(w))|}{|f(w)|} > 1 + \epsilon\right)\right|.
\]
Let $\epsilon > 0$ be arbitrary. 
For $w \in S$, let $E_\epsilon(w)$ be the event that $|f(\Psi(w))| > (1+\epsilon)|f(w)|$. Also, let $p_\epsilon = \inf_{w\in S}\mathbb{P}(E_\epsilon(w)) > 0$ from the hypotheses. Since $\Psi$ is independent of $W_k$, by laws of conditional expectation, for any $\beta > 0$,
\begin{align*}
\|f(W_{k+1})\|_{\beta}^\beta &= \mathbb{E}[\mathbb{E}[|f(W_{k+1})|^\beta \, \vert \, W_k]] \\
&\geq \mathbb{E}[\mathbb{P}(E_\epsilon(W_k)\vert W_k) \mathbb{E}[|f(W_{k+1})|^\beta \, \vert \, W_k,E_{\epsilon}(W_k)]] \\
&\geq \mathbb{E}[\mathbb{P}(E_{\epsilon}(W_{k})\vert W_{k})(1+\epsilon)^{\alpha}|f(W_{k})|^{\beta}]\\
&\geq p_\epsilon (1+\epsilon)^\beta \|f(W_k)\|_\beta^\beta.
\end{align*}
For any $\beta > \alpha$, $p_\epsilon(1+\epsilon)^\beta > 1$,
and hence $\|f(W_k)\|_\beta$ diverges as $k\to\infty$. By the $f$-norm ergodic theorem, $|f(W_\infty)|^{\beta}$ cannot be integrable; if it were, then $\|f(W_k)\|_\beta$ would converge to $\|f(W_\infty)\|_\beta < +\infty$. 
\end{proof}

The arguments of \cite{alsmeyer2016stationary} almost imply Theorem~\ref{thm:Lipschitz}, but are incompatible with the conditions that $M_\Psi$ is non-negative, and $k_\Psi$ can be zero. Instead, more subtle arguments are required; for these, we draw inspiration from \cite{goldie1991implicit,goldie1996perpetuities}. 

\begin{proof}[Proof of Theorem~\ref{thm:Lipschitz}]
Since the probability measure of $\Psi$ is non-atomic, \cite[Theorem 2.1]{alsmeyer2003harris} implies that $\{W_k\}_{k=0}^{\infty}$ is Harris recurrent. Together with 
\cite[Theorem 3.2]{alsmeyer2003harris}, we find that $\{W_k\}_{k=0}^{\infty}$ is geometrically ergodic.
That $W_{\infty}$ satisfies the distributional fixed point equation $W_{\infty} \eqd \Psi(W_{\infty})$ is precisely Letac's principle \cite[Theorem 2.1]{goldie1991implicit}. 

We shall begin by proving (2). Recall that\footnote{This is readily shown using L'H\^{o}pital's rule.} $\log x = \lim_{s \to 0^+} s^{-1}(x^s - 1)$, and so
\[
\lim_{s\to 0^+}\frac{\mathbb{E}K_{\Psi}^s - 1}{s} = \mathbb{E}\log K_{\Psi} < 0.
\]
Therefore, there exists some $s > 0$ such that $\|K_{\Psi}\|_s < 1$. Using H\"{o}lder's inequality, one finds that $\|K_{\Psi}\|_{\beta-\epsilon} < 1$ for any $\epsilon > 0$. Likewise, since $K_{\Psi} < 1$ with positive probability, $k_{\Psi} < 1$ with positive probability also, and hence, $k_{\Psi} > 1$ with positive probability. Therefore, there exists $r > 0$ such that $\|k_{\Psi}\|_r > 1$, and so $\|k_{\Psi}\|_{\alpha+\epsilon} > 1$ for any $\epsilon > 0$.
We now consider similar arguments to the proof of Lemma~\ref{lem:Abstract}. Since $\{W_k\}_{k=0}^{\infty}$ is geometrically ergodic, by the $f$-norm ergodic theorem, for any $\gamma > 0$, $\|W_{\infty}\|_\gamma$ is finite if and only if $\|W_k\|_{\gamma}$ is bounded in $k$. Letting $0 < \epsilon < \delta$, for each $k=0,1,\dots$,
\[
\|W_{k+1}\|_{\beta-\epsilon} \leq \|K_{\Psi}\|_{\beta - \epsilon} \|W_k\|_{\beta - \epsilon} + \|K_{\Psi}\|_{\beta-\epsilon}\|w^{\ast}\| + \|\Psi(w^{\ast})\|_{\beta - \epsilon}.
\]
Note that $\beta < \alpha$, since $k_{\Psi} \leq K_{\Psi}$ almost surely. Therefore, $\|\Psi(w^{\ast})\|_{\beta -\epsilon} < \infty$, and since $\|K_{\Psi}\|_{\beta - \epsilon} < 1$, $\|W_k\|_{\beta-\epsilon}$ is bounded and $\|W_{\infty}\|_{\beta-\epsilon}$ is finite. By Markov's inequality, $\mathbb{P}(\|W_{\infty}\| > t) \leq \|W_{\infty}\|_{\beta-\epsilon}^{\beta-\epsilon} t^{-\beta+\epsilon}$ for all $t > 0$. 
On the other hand, for each $k=0,1,\dots,$
\[
\|W_{k+1}\|_{\alpha + \frac12\epsilon} \geq \|k_{\Psi}\|_{\alpha + \frac12\epsilon} \|W_k\|_{\alpha + \frac12\epsilon} - \|k_{\Psi}\|_{\alpha + \frac12\epsilon} \|w^{\ast}\| - \|\Psi(w^{\ast})\|_{\alpha+\frac12\epsilon},
\]
and so $\|W_\infty\|_{\alpha+\frac12\epsilon}$ is necessarily infinite. By Fubini's theorem, 
\[
\|W_{\infty}\|_{\alpha+\frac12\epsilon}^{\alpha+\frac12\epsilon} = (\alpha+\tfrac12\epsilon)\int_0^\infty t^{\alpha+\frac12\epsilon-1}\mathbb{P}(\|W_{\infty}\| > 
t) \dd t.
\]
Therefore, we cannot have that $\limsup_{t\to\infty} t^{\alpha+\epsilon}\mathbb{P}(\|W_{\infty}\| > t) = 0$, since this would imply
\[
\|W_{\infty}\|_{\alpha+\frac12\epsilon}^{\alpha+\frac12\epsilon} \leq (\alpha+\tfrac12\epsilon)\int_0^\infty t^{-1-\frac12\epsilon} \dd t < +\infty,
\]
and hence $\limsup_{t\to\infty} t^{\alpha+\epsilon}\mathbb{P}(\|W_{\infty}\| > t) > 0$. Repeating these arguments for $p$ in place of $\beta-\epsilon$ and $\alpha+\tfrac12 \epsilon$ implies statement (3). 

Turning now to a proof of (1), since we have already shown the upper bound, it remains to show that $\mathbb{P}(\|W_\infty\| > t) = \Omega(t^{-\mu})$ for some $\mu > 0$; by Lemma~\ref{lem:TailToWhole} this implies the claimed lower bound. We shall achieve this with the aid of Lemmas~\ref{lem:PowerLowerBound},~\ref{lem:RecurLowerBound}, and~\ref{lem:TailExpLowerBound}. First, since $\mathbb{P}(k_\Psi > 1) > 0$ and $x\mapsto \mathbb{P}(X > x)$ is right-continuous, there exists $\epsilon > 0$ such that $\mathbb{P}(k_\Psi > (1+\epsilon)^2) > 0$. In the sequel, we let $C_{\alpha,\epsilon}$ denote a constant dependent only on $\alpha,\epsilon$, not necessarily the same on each appearance. Let $B_\Psi = k_\Psi\|w^\ast\|+M_\Psi+\|\Psi(w^\ast)\|$. By Lemma~\ref{lem:PowerLowerBound},
\begin{align*}
\|W_{k+1}\|^\alpha &\geq (1+\epsilon)^{-\alpha} (\|W_{k+1}\|+\|\Psi(w^\ast)\|)^\alpha - C_{\alpha,\epsilon}\|\Psi(w^\ast)\|^\alpha \\
&\geq (1+\epsilon)^{-\alpha}\|W_{k+1} - \Psi(w^{\ast})\|^{\alpha} - C_{\alpha,\epsilon}\|\Psi(w^\ast)\|^\alpha \\
&\geq (1+\epsilon)^{-\alpha}(k_\Psi\|W_k\| - k_\Psi\|w^\ast\|-M_\Psi)^{\alpha} - C_{\alpha,\epsilon}\|\Psi(w^\ast)\|^\alpha \\
&\geq (1+\epsilon)^{-2\alpha}k_\Psi^\alpha \|W_k\|^\alpha - C_{\alpha,\epsilon}B_\Psi^\alpha.
\end{align*}
For $\alpha > 0$ and $k=0,1,\dots,$ let $f_k^\alpha(t) = \mathbb{E}[\|W_k\|^\alpha \ind_{\|W_k\|\leq t}]$. Then
\[
f_{k+1}^{\alpha}(t) \geq \mathbb{E}[k_\Psi^\alpha(1+\epsilon)^{-2\alpha} \ind_{\{\|W_k\|\leq t / K_\Psi -\|w^\ast\|\}}] - C_{\alpha,\epsilon}\mathbb{E}B_\Psi^\alpha.
\]
Let $c > 1$ be some constant sufficiently large so that $\mathbb{P}(k_\Psi > (1+\epsilon)^2 \vert K_\Psi \leq c) > 0$. We may now choose $\alpha > 0$ such that $\mathbb{E}[k_\Psi^\alpha(1+\epsilon)^{-2\alpha}\vert K_\Psi \leq c] \mathbb{P}(K_\Psi \leq c) \eqqcolon a > c$. Doing so, we find that
\[
f_{k+1}^{\alpha}(ct) \geq a f_k(t - \|w^{\ast}\|) - C_{\alpha,\epsilon}\mathbb{E}B_\Psi^\alpha,\qquad \mbox{for any }t > 0,\,k \geq 0.
\]
By the $f$-norm ergodic theorem, $f_k^{\alpha}(t) \to f^{\alpha}(t)$ pointwise as $k \to \infty$, where $f^{\alpha}(t) = \mathbb{E}[\|W_\infty\|^\alpha \ind_{\|W_\infty\| \leq t}]$. Therefore,
\[
f^{\alpha}(ct) \geq af(t - \|w^{\ast}\|) - C_{\alpha,\epsilon}\mathbb{E}B_\Psi^\alpha,\qquad \mbox{for any }t > 0.
\]
By Lemma~\ref{lem:RecurLowerBound}, this implies that there exists some $0 < \gamma <\alpha$ such that $\liminf_{t\to\infty} t^{-\gamma} f^{\alpha}(t) > 0$, which from Lemma~\ref{lem:TailExpLowerBound}, implies that $\mathbb{P}(\|W_\infty\|\geq t) = \Omega(t^{-\alpha(\alpha-\gamma)/\gamma})$.
\end{proof}

\begin{lemma}
\label{lem:TailToWhole}
Suppose that $\mathbb{P}(X>x)\geq Cx^{-\alpha}$ for all $x\geq x_{0}$. Then there exists $c>0$ such that $\mathbb{P}(X>x)\geq c(1+x)^{-\alpha}$ for all $x \geq 0$.
\end{lemma}
\begin{proof}
Evidently, $\mathbb{P}(X>x)\geq C(1+x)^{-\alpha}$ for $x\geq x_{0}$. Treating the $x\leq x_{0}$ setting, let \[C_{0}=\inf_{x\leq x_{0}}\frac{\mathbb{P}(X>x)}{(1+x)^{\alpha}}.\] Assume $C_{0}=0$. Since $(1+x)^{\alpha}$ is bounded for all $x\geq0$, there exists a sequence $\{x_{n}\}_{n=1}^{\infty}\in[0,x_{0}]$ such that $\mathbb{P}(X>x_{n})\to0$. But since $\{x_{n}\}_{n=1}^{\infty}$ there exists some subsequence converging to a point $x\leq x_{0}$, which must satisfy $\mathbb{P}(X>x)=0$, contradicting our hypotheses. Therefore, $C_{0}\neq0$ and the result is shown for $c=\min\{C_{0},C\}$. 
\end{proof}

\begin{lemma}
\label{lem:PowerLowerBound}
For any $\alpha > 1$ and $\epsilon > 0$, there exists $C_{\alpha,\epsilon} > 0$ such that for any $x,y,z\geq 0$, 
\begin{align*}
    z^\alpha &\geq (1+\epsilon)^{-\alpha}(y+z)^\alpha - C_{\alpha,\epsilon} y^\alpha \\
    (x-y)_+^{\alpha} &\geq (1+\epsilon)^{-\alpha} x^\alpha - C_{\alpha,\epsilon} y^\alpha.
\end{align*}
\end{lemma}
\begin{proof}
The second of these two inequalities follows from the first by taking $(x-y)_+ = z$. Since the first inequality is trivially the case when $y = 0$, letting $\rho = (1+\epsilon)^\alpha$, it suffices to show that
\[
\sup_{z \geq 0,\,y > 0} \frac{(y+z)^{\alpha} - \rho z^{\alpha}}{\rho y^{\alpha}} < \infty.
\]
Equivalently, parameterizing $z = L y$ where $L \geq 0$, it suffices that $\sup_{L \geq 0}[(1+L)^\alpha - \rho L^\alpha] < \infty$, which is evidently the case since $(1 + L^{-1})^{\alpha} - \rho < 0$ for sufficiently large $L > 0$. 
\end{proof}

\begin{lemma}
\label{lem:RecurLowerBound}
Let $f(t)$ be an unbounded non-decreasing function. If there exists some $a\geq c>1$ and $b,x,t_0\geq 0$ such that for $t \geq t_0$,
\begin{equation}
\label{eq:FuncLinearRecur}
    f(ct)\geq af(t-x)-b,
\end{equation}
then $\liminf_{t\to\infty}t^{-\gamma}f(t)>0$ for any $\gamma < \frac{\log a}{\log c}$. 
\end{lemma}
\begin{proof}
First, consider the case $x = 0$. Iterating (\ref{eq:FuncLinearRecur}), for any $n=1,2,\dots,$ and $t > t_0$,
\[
f(c^n t) \geq \left(f(t) - \frac{b}{a-1}\right) a^n + \frac{b}{a-1}.
\]
Let $t_1$ be sufficiently large so that $t_1 > t_0$ and $f(t) > \frac{b}{a-1}$ for all $t > t_1$. Then for any $\alpha > 0$, letting $[\alpha]$ denote the largest integer less than or equal to $\alpha$,
\begin{align*}
f(c^{\alpha} t_1) &\geq \left(f(c^{\alpha - [\alpha]} t_1) - \frac{b}{a-1}\right) a^{[\alpha]} + \frac{b}{a-1}, \\
&\geq \left(f(t_1) - \frac{b}{a-1}\right) a^{\alpha - 1} + \frac{b}{a-1}.
\end{align*}
In particular, by choosing $\alpha = \frac{\log t - \log t_1}{\log c}$, $c^\alpha t_1 = t$, and so for any $t > t_1$,
\[
f(t) \geq a^{-\frac{\log t_1}{\log c} - 1} \left(f(t_1) - \frac{b}{a-1}\right) t^{\frac{\log a}{\log c}} + \frac{b}{a-1}.
\]
Now, let $t_2$ be sufficiently large so that $t_2 \geq t_1$ and
\[
a^{-\frac{\log t_1}{\log c} - 1} \left(f(t_1) - \frac{b}{a-1}\right) t_2^{\frac{\log a}{\log c}} \geq \frac{b}{a - 1}.
\]
Then for $t \geq t_2$ and $C = 2 a^{-\frac{\log t_1}{\log c} - 1} (f(t_1) - \frac{b}{a-1})$,
$
f(t) \geq C t^{\frac{\log a}{\log c}},
$
and the result follows.
Now, suppose that $x > 0$. Let $0 < \epsilon < 1$ and take $t_1 = \max\{t_0, x / \epsilon\}$ so that $t - x \geq (1-\epsilon) t$ for all $t \geq t_1$. Then, for all $t \geq t_1$, 
\[
f\left(\frac{c}{1-\epsilon} \cdot t\right) \geq a f(t) - b.
\]
If we can show the conclusion for the case where $x = 0$, then for $\gamma < \frac{\log a}{\log c}$,
\[
\liminf_{t \to \infty} t^{-\gamma \cdot \frac{\log c}{\log c + |\log(1-\epsilon)|}} f(t) > 0.
\]
Since $\epsilon > 0$ was arbitrary, the result follows. 
\end{proof}

\begin{lemma}
\label{lem:TailExpLowerBound}
Let $X$ be a non-negative random variable and $\alpha > 0$. If there exists some $0 < \gamma < \alpha$ such that
\[
\liminf_{t\to\infty} t^{-\gamma} \mathbb{E}[X^{\alpha} \ind_{X \leq t}] > 0,
\]
then there is some $C, t_0 > 0$ such that for $t \geq t_0$,
\[
\mathbb{P}(X \geq t) \geq C t^{-\frac{\alpha}{\gamma}(\alpha - \gamma)}.
\]
\end{lemma}
\begin{proof}
By taking $t$ to be sufficiently large, there exists $c_1 > 0$ such that
\begin{align*}
c_1 t^{\gamma} \leq \mathbb{E}[X^{\alpha} \ind_{X \leq t}] &= \alpha \int_0^\infty \left(\int_0^\infty \ind_{u \leq t} \dd \mathbb{P}_X(u)\right) \ind_{v \leq u} v^{\alpha - 1} \dd v \\
&=\alpha\int_0^t \mathbb{P}(v \leq X \leq t) v^{\alpha -1 } \dd v \leq \alpha \int_0^t \mathbb{P}(X \geq v) v^{\alpha - 1} \dd v.
\end{align*}
On the other hand, observe that for any $b, t > 1$,
\begin{align*}
\alpha \int_0^{bt} \mathbb{P}(X \geq v) v^{\alpha -1 } \dd v &\leq \alpha \int_0^t v^{\alpha - 1} \dd v + \alpha \mathbb{P}(X \geq t)\int_t^{bt} v^{\alpha - 1} \dd v \\
&= t^{\alpha} [1 + \mathbb{P}(X \geq t)(b^\alpha - 1)].
\end{align*}
Therefore,
\[
\frac{c b^\gamma t^{\gamma - \alpha} - 1}{b^\alpha - 1} \leq \mathbb{P}(X \geq t).
\]
Choosing $b = (2 / c)t^{(\alpha - \gamma)/\gamma}$ such that $c b^\gamma t^{\gamma - \alpha} - 1 = 1$, the lemma follows.
\end{proof}

\end{document}